\documentclass{article}
\usepackage{amsthm}
\usepackage{microtype}
\usepackage{graphicx}

\usepackage{booktabs} % for professional tables
\usepackage{multirow}
\usepackage{thm-restate}

\usepackage{booktabs} % For better table formatting
\usepackage{siunitx}
\newcommand{\modelname}[1]{{\fontfamily{qpl}\selectfont {\small #1}}}

\usepackage{tabularx}
\usepackage{bm}

\newlength{\mybarwidth}
\usepackage{float}
\usepackage[accepted]{icml2024}

\usepackage{amsmath}
\usepackage{amssymb}

\usepackage{mathtools}

\usepackage{dblfloatfix}
\usepackage{comment}

\usepackage{graphicx}

\usepackage{caption}
\usepackage{subcaption}
\usepackage{bbm}
\usepackage{comment}
\usepackage{wrapfig}

\newcommand{\bx}{\mathbf{x}}

\newcommand{\bh}{\mathbf{h}}

\usepackage{amsthm}
\theoremstyle{plain}
\newtheorem{theorem}{Theorem}[section]

\newtheorem{corollary}[theorem]{Corollary}
\theoremstyle{definition}

\theoremstyle{remark}

\usepackage[textsize=tiny]{todonotes}
\usepackage{tikz}
\usetikzlibrary{plotmarks}
\usepackage{pgfplots}
\usepackage{tableaucolors}
\pgfplotsset{compat=1.18}

\newcommand{\Sone}{C6}
\newcommand{\Stwo}{C2}
\newcommand{\Sthree}{C9}
\newcommand{\Sfour}{C3}

\newcommand{\Sfifty}{C0}

\newcommand{\stwoaunbiasedthm}{The \modelname{Some-to-All} estimator is an unbiased estimator of Eq. \eqref{eq:miselbo}.}

\newcommand{\newstwoaunbiasedcorollary}{The expected value of the \modelname{Some-to-All} estimator is a lower bound on the marginal log-likelihood, \begin{align*}
    \mathbb{E}\left[\widetilde{\mathcal{L}}_\text{\modelname{S2A}}\right] \leq \log p_\theta(x).
\end{align*}}

\newcommand{\stwoslowerboundmiselbothm}{The expected value of the \modelname{Some-to-Some} estimator is a lower bound on MISELBO, i.e.,
    \begin{align*}
        \mathbb{E}\left[
    \widetilde{\mathcal{L}}_\text{\modelname{S2S}}
        \right]\leq \mathcal{L}_\textnormal{MIS}.
    \end{align*}}

\newcommand{\stwoslowerboundmarglikcorollary}{The expected value of the \modelname{Some-to-Some} estimator is a lower bound on the marginal log-likelihood,
    \begin{align*}
        \mathbb{E}\left[
    \widetilde{\mathcal{L}}_\text{\modelname{S2S}}
        \right]\leq \log p_\theta(x).
    \end{align*}}

\icmltitlerunning{Efficient Mixture Learning in Black-Box Variational Inference}

\newtheoremstyle{nonitalic}% name of the style to be used
  {3pt}% Space above
  {3pt}% Space below
  {}% Body font
  {}% Indent amount
  {\bfseries}% Theorem head font
  {:}% Punctuation after theorem head
  {.5em}% Space after theorem head
  {}% Theorem head spec (can be left empty, meaning ‘normal’)

\theoremstyle{nonitalic}

\pgfplotsset{
    symlog style/.style={
        y coord trafo/.code=\pgfmathparse{symlog(\pgfmathresult)},
        y coord inv trafo/.code=\pgfmathparse{symexp(\pgfmathresult)},
        yticklabel style={/pgf/number format/.cd, fixed},
        yticklabel=\pgfmathparse{symexp(\tick)}\pgfmathresult,
    },
    /pgf/declare function={
        symlog(\x)= (\x > 0) * ln(max(\x,1)) + (\x < 0) * -ln(max(-\x,1));
        symexp(\x)= (\x > 0) * exp(\x)   + (\x < 0) * -exp(-\x);
    }
}
\usepgfplotslibrary{fillbetween}

\pgfplotsset{compat=1.17}
\usetikzlibrary{patterns, shapes.geometric}
\tikzset{
        /tikz/every even column/.append style={column sep=1em},
    }
\pgfplotsset{
    every axis/every ticks/.append style={fontsize=\small},
    every axis legend/.append style={
            font=\sc\small,
            {/tikz/every even column/.append style={column sep=1em}},
        }
}
 
\usepgfplotslibrary{units} % Allows to enter the units nicely

\usepackage{hyperref}
\usepackage[capitalize,noabbrev]{cleveref}

\begin{document}

\twocolumn[
\icmltitle{Efficient Mixture Learning in Black-Box Variational Inference}

% It is OKAY to include author information, even for blind
% submissions: the style file will automatically remove it for you
% unless you've provided the [accepted] option to the icml2023
% package.

% List of affiliations: The first argument should be a (short)
% identifier you will use later to specify author affiliations
% Academic affiliations should list Department, University, City, Region, Country
% Industry affiliations should list Company, City, Region, Country

% You can specify symbols, otherwise they are numbered in order.
% Ideally, you should not use this facility. Affiliations will be numbered
% in order of appearance and this is the preferred way.

\icmlsetsymbol{equal}{*}

\begin{icmlauthorlist}
\icmlauthor{Alexandra Hotti}{equal,yyy,comp,klarna}
\icmlauthor{Oskar Kviman}{equal,yyy,comp}
\icmlauthor{Ricky Molén}{yyy,comp}
\icmlauthor{Víctor Elvira}{xxx}
\icmlauthor{Jens Lagergren}{yyy,comp}
%\icmlauthor{}{sch}
%\icmlauthor{}{sch}
\end{icmlauthorlist}

\icmlaffiliation{yyy}{KTH Royal Institute of Technology}
\icmlaffiliation{comp}{Science for Life Laboratory}
\icmlaffiliation{klarna}{Klarna}
\icmlaffiliation{xxx}{University of Edinburgh}

\icmlcorrespondingauthor{Alexandra Hotti}{hotti@kth.se}
\icmlcorrespondingauthor{Oskar Kviman}{okviman@kth.se}

% You may provide any keywords that you
% find helpful for describing your paper; these are used to populate
% the "keywords" metadata in the PDF but will not be shown in the document
\icmlkeywords{Machine Learning, ICML}
 
\vskip 0.3in
]

 % Needed for pgfplots
%\usepackage{pgfplots}% not approved
%\usepackage{pgfplotstable}% not approved
%\usepackage{tableaucolors}% not approved

% this must go after the closing bracket ] following \twocolumn[ ...

% This command actually creates the footnote in the first column
% listing the affiliations and the copyright notice.
% The command takes one argument, which is text to display at the start of the footnote.
% The \icmlEqualContribution command is standard text for equal contribution.
% Remove it (just {}) if you do not need this facility.

%\printAffiliationsAndNotice{}  % leave blank if no need to mention equal contribution
\printAffiliationsAndNotice{\icmlEqualContribution} % otherwise use the standard text.
% Despite these successes, scaling state-of-the-art (SOTA) architectures to accommodate a large number of mixture components has remained a daunting challenge. 

\begin{abstract}

Mixture variational distributions in black box variational inference (BBVI) have demonstrated impressive results in challenging density estimation tasks.
However, currently scaling the number of mixture components can lead to a linear increase in the number of learnable parameters and a quadratic increase in inference time due to the evaluation of the evidence lower bound (ELBO). Our two key contributions address these limitations. First, we introduce the novel Multiple Importance Sampling Variational Autoencoder (\modelname{MISVAE}), which amortizes the mapping from input to mixture-parameter space using one-hot encodings. Fortunately, with \modelname{MISVAE}, each additional mixture component incurs a negligible increase in network parameters. Second, we construct two new estimators of the ELBO for mixtures in BBVI, enabling a tremendous reduction in inference time with marginal or even \textit{improved} impact on performance. Collectively, our contributions enable scalability to hundreds of mixture components and provide superior estimation performance in shorter time, with fewer network parameters compared to previous Mixture VAEs. Experimenting with \modelname{MISVAE}, we achieve astonishing, SOTA results on MNIST. Furthermore, we empirically validate our estimators in other BBVI settings, including Bayesian phylogenetic inference, where we improve inference times for the SOTA mixture model on eight data sets. 
\end{abstract}

\section{Introduction}

Recent advancements in variational inference (VI) have focused on enhancing performance through more sophisticated network architectures, formulation of flexible priors and variational posteriors, and the exploration of alternative formulations of the evidence lower bound (ELBO), the typical objective function in VI. Competitive developments include normalizing flows (NFs; \citet{rezende2015variational, papamakarios2021normalizing}), hierarchical models \cite{burda2015importance, sonderby2016ladder,vahdat2020nvae}, autoregressive models \cite{van2016pixel}, the VampPrior \cite{tomczak2018vae}, and the importance weighted ELBO (IWELBO;  \citet{burda2015importance}).

Lately, using mixture models as variational distributions has garnered increased attention \citep{nalisnick2016approximate, kucukelbir2017automatic, morningstar2021automatic, kviman2022multiple, kviman2023cooperation}. Specifically, \citet{kviman2022multiple} developed a formulation of the ELBO for uniformly weighted mixtures inspired by multiple importance sampling (MIS; see \citet{elvira2019generalized} for a review), termed MISELBO,\footnote{An alternative naming of this lower bound emerges in the work of \citet{morningstar2021automatic}, where the sampling scheme is interpreted as stratified sampling. This perspective leads to the names SELBO or SIWELBO.} \begin{align}
\label{eq:miselbo}
    &\mathcal{L}_\textnormal{MIS} =\frac{1}{A}\sum_{a=1}^A \mathbb{E}_{q_{\phi_a}(z_a|x)}\left[
    \log \frac{p_\theta(x, z_{a})}{\frac{1}{A}\sum_{a^\prime=1}^A q_{\phi_{a^\prime}}(z_{a}|x)}
    \right],
\end{align} where $z_a$ is a latent variable, $x$ is observed data, \(\theta\) represents the parameters of the generative model \(p_\theta(x, \cdot)\), $A$ is the number of mixture components, and \(\phi_{a}\) denotes the variational parameters of the \(a\)-th mixture component \(q_{\phi_a}(\cdot|x)\).

Mixtures are distinguished by their simple yet expressive nature, as well as their theoretical foundation. In the BBVI setting, they have achieved state-of-the-art (SOTA) results in applications like image processing and phylogenetics \cite{kviman2023cooperation, kviman2023improved}. However, computational complexities (parameter costs and inference times) hinder algorithm developers from utilizing a large $A$, ultimately leaving the full potential (e.g., their universal approximator property \cite{kostantinos2000gaussian}) untapped.

In BBVI-based mixture learning, the number of learnable parameters typically increases linearly with $A$ at an unpractical rate. For example, in \citet{kviman2022multiple, kviman2023cooperation}, a naive approach is used where each new component allocates a separate encoder network, while in \citet{kviman2023improved}, there is one Bayesian network per component. Moreover,

\begin{figure}[H]
\centering
\begin{minipage}{0.32\textwidth}
\centering
\begin{tikzpicture}[trim axis left, trim axis right]
  \begin{axis}[
    xlabel={$A$},
    ylabel={$-\log p_{\theta}(x)$},
    major grid style={dotted,black},
    ymajorgrids=true,
    width=0.999\linewidth,
    title = FashionMNIST,
    ymin=220.5, ymax=226,
  ]
    \addplot[color=C1, mark options={solid}, mark=*, very thick, mark options={solid, fill opacity=0.69}, name path=mean] coordinates {
      (1,225.0730738) (10,223.4221537) (20,222.815659)
      (50,222.5405424) (75,222.2620291015625) (100,222.14271171875)
      (200,221.6042012) (400,221.2495875)
    };
    \addplot[name path=upper, draw=none] coordinates {
    (1,225.0730738+ 0.0685) 
    (10,223.4221537 + 0.132) 
    (20,222.815659 + 0.0582)
    (50,222.5405424 + 0.0671) 
    (75,222.2620291015625 + 0.0909) 
    (100,222.14271171875 + 0.0600)
    (200,221.6042012 + 0.0561) 
    (400,221.2495875 + 0)
    };
    \addplot[name path=lower, draw=none] coordinates {
    (1,225.0730738- 0.0685) 
    (10,223.4221537 - 0.132) 
    (20,222.815659 - 0.0582)
    (50,222.5405424 - 0.0671) 
    (75,222.2620291015625 - 0.0909) 
    (100,222.14271171875 - 0.0600)
    (200,221.6042012 - 0.0561) 
    (400,221.2495875 - 0)
    };
    \addplot[\Sfour!40, fill opacity=0.5] fill between[of=upper and lower];
    \addplot[dashed, color=C2, very thick] coordinates {(0,0)};
\draw [dashed, color=C2, very thick] (axis cs:-100,222.38) -- (axis cs:500,222.38);
  \end{axis}
\end{tikzpicture}
\end{minipage}
\begin{minipage}{0.32\textwidth}
\centering
\begin{tikzpicture}[trim axis left, trim axis right]
  \begin{axis}[
    xlabel={$A$},
    ylabel={$-\log p_{\theta}(x)$},
    title = MNIST,
    major grid style={dotted,black},
    ymajorgrids=true,
    width=0.999\linewidth,
    ytick={74,75,76,77,78,79,80},
    ymin=73.5, ymax=80,
  ]
    \addplot[color=C1, mark options={solid}, mark=*, very thick, mark options={solid, fill opacity=0.69}] coordinates {
      (1,79.64566655) (10,78.01734403) (20,77.31235078)
      (50,76.66027326) (75,76.2737875) (100,76.03476406)
      (200,75.43070625) (400,74.98859375) (600, 74.39796909179688) (800, 74.07091489257813) (1000, 74.12734091796875)
    };
    \addplot[dashed, color=C9, very thick] coordinates {(0,0)};
\draw [dashed, color=C9, very thick] (axis cs:-100,76.93) -- (axis cs:120000,76.93);

  \end{axis}
\end{tikzpicture}
\end{minipage}
\fbox{
\begin{tabular}{@{}l@{\hspace{10pt}}l@{\hspace{15pt}}l@{}}
\tikz\draw[C1,fill=C1!69] (0,0) circle (.7ex); & \modelname{MISVAE} w/ $\mathrm{S2A}$ & \\
\raisebox{0.5ex}{\tikz\draw[very thick, C9] (0,0) -- (0.75,0) [dash pattern=on 4pt off 4pt];} & \modelname{CR-NVAE} & \\
\raisebox{0.5ex}{\tikz\draw[very thick, \Stwo] (0,0) -- (0.75,0) [dash pattern=on 4pt off 4pt];} & \modelname{Composite} \modelname{SEMVAE} & \\
\end{tabular}}

\caption{\textbf{SOTA Performance with Small and Efficient Networks:} NLL values for \modelname{MISVAE} trained with the $\mathrm{S2A}$ estimator with $S=1$ and a gradually increasing $A$.}
%The number of network parameters and the average training time per epoch for $A=200,800$ are presented in Table \ref{tab:NLL SOTA} for MNIST. For FashionMNIST for each configuration of S and A, we ran three independent training runs, and report the average NLL score (solid) obtained on the test set, along with the standard deviation (opaque) computed over the training runs. Note that we for FashionMNIST used a single seed for the setting where $A=400$.
%\vspace{-5pt}
\label{fig:big_smnist}
\end{figure}
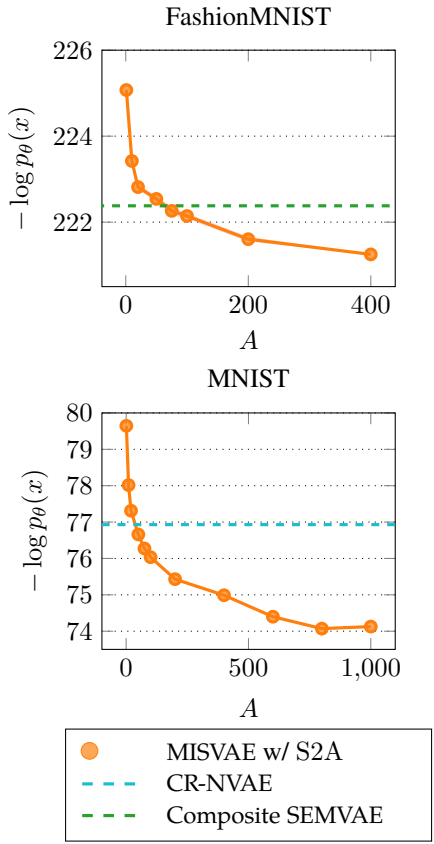
by inspecting Eq. \eqref{eq:miselbo}, it is clear that the evaluation of the MISELBO objective, and thus the inference time, scales quadratically with $A$. 

We make two contributions in this work, each addressing the aforementioned computational complexities. First, we introduce the Multiple Importance Sampling VAE (\modelname{MISVAE}), a novel VAE architecture that efficiently amortizes the mapping from data to mixture parameters (see Fig. \ref{fig:misvae-architecture}). With our one-hot-encoding-based parameterization strategy, all network weights in a single encoder are shared among the $A$ mixture components. This is a novel construction, as, previously, either no \cite{kviman2022multiple, kviman2023cooperation} or only a subset \cite{nalisnick2016approximate} of the encoder weights have been shared. \modelname{MISVAE} is described in detail in Sec. \ref{sec:misvae}.

Our second contribution is inspired by the plethora of techniques for sampling from mixture models, from the MIS literature \cite{elvira2019generalized}. To make the evaluation of MISELBO objective more effective, we extend two established MIS schemes to develop two novel estimators of the MISELBO: the \modelname{Some-to-All} ($\mathrm{S2A}$) and \modelname{Some-to-Some} ($\mathrm{S2S}$) estimators. Both estimators sample a subset of $S<A$ unique components, from which the latent variables are subsequently simulated from. This results in estimations of a subset of the expectations in Eq. \eqref{eq:miselbo}. The two estimators differ, however, in their formulation of the denominator in Eq. \eqref{eq:miselbo} and, thus, in their theoretical properties. In Sec. \ref{sec:estimators}, we clearly explain how to implement the estimators, how they relate to popular MIS schemes, provide their respective time complexities, and give robust theoretical justifications of both estimators. 

We have constrained our work to uniformly weighted mixtures. This is justified by existing analyses in the BBVI literature. Specifically, \citet{morningstar2021automatic} observed that inferring parameters for \textit{weighted} mixture components often leads to mode collapse. To mitigate this issue, they suggested using the Importance Weighted ELBO (IWELBO; \citet{burda2015importance}). However, using the IWELBO objective presents other challenges. Notably, it does not allow for training VAEs using the KL warm-up scheme, which is crucial for achieving SOTA NLL results (see e.g., \citet{tomczak2018vae}). Additionally, this approach can adversely affect the learning of encoder nets  \cite{rainforth2018tighter}.

The $\mathrm{S2A}$ and $\mathrm{S2S}$ estimators are applicable to any BBVI mixture problem (i.e., not constrained to VAEs), and, as we demonstrate in various BBVI scenarios in Sec. \ref{sec:experiments}, they can heavily decrease the parameter inference time. For VAEs, when paired with \modelname{MISVAE}, we can push the limits of $A$ in order to achieve astonishing marginal log-likelihood results on MNIST and FashionMNIST (see Fig. \ref{fig:big_smnist}).

To summarize, our contributions are
\begin{itemize}
    \item \textbf{\modelname{MISVAE}:}  We introduce \modelname{MISVAE}. A novel Mixture VAE architecture that efficiently maps data to mixture parameters, significantly improving the scalability w.r.t. the number of mixture components, $A$.
    \item \textbf{\modelname{Some-to-Some}:} We propose $\mathrm{S2S}$, a novel estimator of MISELBO (Eq. \eqref{eq:miselbo}). This estimator enables enhanced performance relative to \modelname{MISVAE} by allowing an increase in $A$, while preserving the same inference time per epoch.
    \item \textbf{\modelname{Some-to-All}:} We introduce $\mathrm{S2A}$, which we prove to be an unbiased estimator of MISELBO for any $S<A$. This approach makes it possible to increase the total number of mixtures $A$ with only a small additional computational burden.

\end{itemize}

 \section{Related Work}

Mixture VAEs can be traced back to \citet{nalisnick2016approximate}, who employed a Gaussian mixture model with a Dirichlet prior on the mixture weights, and inferred these weights through a neural network mapping from the data, \(x\), to the simplex. Their encoder employed $A$ separate mappings from a hidden layer in the encoder to the different mixture parameters. As such, their architecture scales poorly to large $A$ in terms of number of network parameters. Furthermore, \citet{roeder2017sticking} utilized a weighted mixture ELBO, for training a VAE using stop gradients and different sampling strategies. Yet, their application was limited to a toy dataset.

Drawing inspiration from the MIS literature, \citet{kviman2022multiple} introduced the concept of the MISELBO, offering a straightforward method for evaluating the mixture ELBO. However, in their approach, individual mixture components were trained separately and aggregated only during the evaluation of the MISELBO objective. As a result, each component tended to gravitate towards the mode of the posterior, rather than all components collaboratively covering all regions of the posterior. Expanding on these concepts, \citet{kviman2023cooperation} devised methods for the joint training of all mixture components, which were subsequently applied by \citet{kviman2023improved}. Collectively, achieving SOTA performance on datasets such as MNIST, FashionMNIST, and a range of phylogenetic datasets. Nevertheless, despite their impressive empirical performance, scaling these architectures to variational mixtures with more than ten components posed significant computational challenges. Here, we address this limitation, enabling the full potential of variational mixtures to be realized by significantly reducing the computational demands of scaling to a larger number of components. The idea of increasing the number of mixture components for variance reduction while limiting the computational complexity in MIS was introduced  in \cite{elvira2015efficient} and further developed in \cite{elvira2016heretical,elvira2016multiple}.

 \section{Background}
\label{sec:background}

\paragraph{Estimating NLL} The estimate of the IWELBO,
\begin{equation}
\label{eq:iwelbo}
    \mathcal{L}^L_\textnormal{IWELBO} = \mathbb{E}_{q_{\phi}(z|x)}\left[
    \log \frac{1}{L}\sum_{\ell=1}^L \frac{p_\theta(x, z_{\ell})}{q_{\phi}(z_{\ell}|x)}
    \right],
\end{equation}
where $L$ is the number of importance samples, is often used to estimate the marginal log-likelihood, $\log p_\theta(x)$. The \textit{negative log-likelihood} (NLL) refers to $\text{-} \log p_\theta(x)$. It is possible to estimate the NLL by using an importance-weighted version of MISELBO \cite{kviman2022multiple}.

\paragraph{MIS} In the field of importance sampling (IS), MIS refers to techniques with more than one proposal/importance sampler \cite{elvira2021advances}. MIS inherits strong theoretical guarantees and many methodological developments have been made \cite{Veach95,Owen00,sbert2022generalizing}. In this line of research, we rely on  MIS schemes where the sampling is done either from a mixture or by deterministically choosing the proposal from a set of mixture components. In both cases, the weights are constructed in a way that reduces variances of the IS estimators (see \citet{elvira2019generalized} for more details).

Due to the logarithm in the ELBO, many theoretical insights gathered in MIS do not generalize to VI. However, recognizing that MISELBO is an expectation to be estimated by a mixture establishes clear connections between popular mixture sampling techniques, as described in \citet{elvira2019generalized}, and the estimators used in BBVI mixture learning.

\paragraph{Bayesian phylogenetics} In Bayesian phylogenetic inference, the posterior distribution over branch lengths and tree topologies is approximated jointly, given the observed sequence data (e.g., DNA). In Appendix \ref{app:vbpi}, we define the posterior and give more details on the generative model. 

Many recent works have applied modern machine learning techniques to Bayesian phylogenetics \cite{zhang2018advances, zhangvbpinf, moretti2021variational, zhang2023learnable, zhou2023phylogfn}. Notably, \citet{kviman2023improved} constructed a mixture of variational phylogenetic posterior approximations, achieving SOTA results.

However, in \citet{kviman2023improved}, the inference time scales poorly with $A$, making it infeasible to learn mixture models with many components. In Sec. \ref{sec:experiments}, we instead apply our two new estimators for learning the mixture parameters, decreasing the computational costs of the SOTA BBVI method. This takes the field closer to realistic application of mixture models in Bayesian phylogenetic and  domains with even larger state spaces. These applications have traditionally been viewed as computationally demanding and, in practice, considered intractable within a Bayesian framework (e.g., species-tree reconciliation \cite{aakerborg2009simultaneous})

\section{Efficient Estimation of MISELBO}
\label{sec:estimators}
We now walk through the three approaches we consider for estimating Eq. \eqref{eq:miselbo}, namely the \modelname{All-to-All}, and our two novel estimators, the \modelname{Some-to-All}, and \modelname{Some-to-Some} estimators.

\paragraph{\modelname{All-to-All}} When implementing the \modelname{All-to-All} (\modelname{A2A}) estimator, a single latent variable is sampled from each of the $A$ available components, resulting in
\begin{equation}    \widetilde{\mathcal{L}}_\textnormal{\modelname{A2A}} = \frac{1}{A}\sum_{a=1}^A \log \frac{p_{\theta}(x|z_{a})p_\theta(z_{a})}{\frac{1}{A}\sum_{{a^\prime}=1}^A q_{\phi_{a^\prime}}(z_a|x)}, 
\end{equation}
where $z_{a} \sim q_{\phi_{a}(z_a|x)}$. 

The computational cost of this estimator is proportional $A\times A$ and it connects to the N3 scheme in \citet{elvira2019generalized} since all $A$ components are used in the simulation of the $S=A$ samples and also appear in the denominator.

\paragraph{\modelname{Some-to-All}}

There are $A$ components in total, and for a given data point (or batch) we sample a subset, $\Phi$, of $S$ unique components (without replacement). No component is more likely to be selected \textit{a priori}, and so we can consider sampling the subsets from a uniform distribution over all ${A \choose S}$ possible subsets, $\varphi(\Phi)$ (see Appendix \ref{app:exp_values}). 

Then, by obtaining $\Phi\sim \varphi(\Phi)$, we construct the \modelname{Some-to-All} (\modelname{S2A}) estimator
\begin{equation}
    \widetilde{\mathcal{L}}_\text{\modelname{S2A}} := \frac{1}{S}\sum_{s=1}^S \log \frac{p_{\theta}(x|z_s)p_\theta(z_s)}{\frac{1}{A}\sum_{a=1}^A q_{\phi_a}(z_s|x)}, 
\end{equation}
where $z_s \sim q_{\phi_{s}(z_s|x)}$ for all $\phi_s \in \Phi$. 

This estimator is linked to the R3 scheme in \citet{elvira2019generalized} since in both cases a subset of components is used to simulate the samples, while the unweighted mixture of all components appear in the denominator. The difference is that the R3 scheme samples exactly $S=A$ samples by selecting the components with multinomial resampling with replacement, while our \modelname{S2A} estimator samples $S<A$ samples by selecting the components with multinomial resampling without replacement instead. 

A beneficial property of the S2A estimator is that it allows for sampling mixture components without replacement, resulting in a lower variance gradient estimator compared to when sampling with replacement. Moreover, the computational cost of the \modelname{S2A} estimator is $S\times A$ and it is an unbiased estimator of Eq. \eqref{eq:miselbo}. 
\begin{theorem}
    \label{stwoaunbiasedcorollary}
\stwoaunbiasedthm  
\end{theorem}
\begin{proof}
    See Appendix \ref{app:exp_values}.
\end{proof}

Furthermore, its expectation is a lower bound on the marginal log-likelihood.

\begin{corollary}\label{stwoalower_bound}
\newstwoaunbiasedcorollary    
\end{corollary}

We leverage the unbiased property of \modelname{S2A} to substantially reduce the complexity involved in estimating the MISELBO objective. The computational cost of this estimator is proportional to $S\times A$, however, given that it is unbiased, we can keep $S$ small and instead increase $A$. 

Although the focus of our work is on uniformly weighted components, we generalize Theorem \ref{stwoaunbiasedcorollary} to hold for arbitrary mixture weights.

\begin{theorem}
\label{w_thm}
    The Some-to-All estimator is an unbiased estimator of MISELBO for arbitrary mixture weights.
\end{theorem}
\begin{proof}
    See Appendix \ref{app:weighted_proof}.
\end{proof}

\paragraph{\modelname{Some-to-Some}} The next estimator is inspired by the \textit{a priori} partitioning approach in \citet[Section 7.2]{elvira2019generalized}. For a given data point, we, again obtain a $\Phi \sim \varphi(\Phi)$, where $|\Phi| = S$. In contrast to the \modelname{S2A} estimator, we here only evaluate the simulated latents on this subset of components, and we get the \modelname{Some-to-Some} (\modelname{S2S}) estimator

\begin{equation}
    \widetilde{\mathcal{L}}_\text{\modelname{S2S}} := \frac{1}{S}\sum_{s=1}^S \log \frac{p_{\theta}(x|z_s)p_\theta(z_s)}{\frac{1}{S}\sum_{\phi_{s'}\in \Phi} q_{\phi_{s'}}(z_s|x)}, 
\end{equation}
where $z_s \sim q_{\phi_{
s}(z_s|x)}$ for all $\phi_s \in \Phi$. The cost of computing this estimator is $S\times S$. Letting $S=1$ is equivalent to inferring the parameters of an ensemble of variational approximations \cite{kviman2022multiple}. 

The \modelname{S2S} estimator also connects with the R2 scheme in \citet{elvira2019generalized} since in both cases a subset of components is used to simulate the samples and this same subset of components appear in the denominator. Again, the difference is that the R2 scheme samples exactly $S=A$ samples by selecting the components with multinomial resampling with replacement, while our \modelname{S2S} samples $S<A$ samples by selecting the components with multinomial resampling without replacement.

\begin{theorem}\label{thm:s2s_lower_bound}
    \stwoslowerboundmiselbothm
\end{theorem}
\begin{proof}
    See the supplementary material.
\end{proof}

\begin{corollary}
    \stwoslowerboundmarglikcorollary
\end{corollary}

The \modelname{S2S} estimator has the smallest computational cost among the three presented here. However, what it gains in speed it trades off in joint inference among the mixture components---a component that is not in $\Phi$ will not affect the inference of $\phi_s\in\Phi$, which should affect cooperation. Yet, given a fixed \( S \), we can improve performance by increasing \( A \) without additional computational burden. This is because the S2S estimator can be viewed as an ensemble of mixtures, where we have access to \( A \) models and select \( S \) components for each instance of the ensemble.

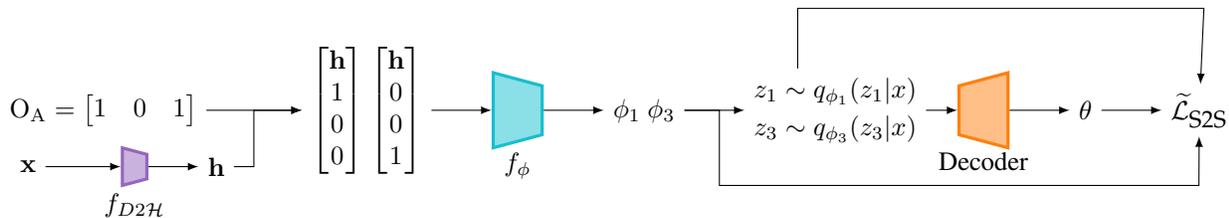
\begin{figure*}[th]
    \centering
\begin{tikzpicture}
    \node[yshift=0.25cm, xshift = -6cm] (onehot) {$\mathrm{O}_\mathrm{A} = \begin{bmatrix} 1 & 0 & 1 \end{bmatrix}$};
    \node[yshift=-0.5cm, xshift = -7cm] (data) {$\bx$};
    \node[yshift=-0.5cm, xshift = -4.5cm] (hidden) {$\bh$};
    \node[yshift=0.25cm, xshift = -2.5cm] (twoonehots) {$\begin{bmatrix} \bh\\1\\ 0 \\ 0 \end{bmatrix}$ $\begin{bmatrix} \bh\\0\\ 0 \\ 1 \end{bmatrix}$};
    
    \coordinate (input); 
    \node[draw,very thick,C4, fill=C4!50, trapezium, trapezium angle=75, shape border rotate=270, minimum
    width=0.5cm, minimum height=0.25cm, right=1cm of data, label=below:$f_{D2\mathcal{H}}$] (fd2h) {};
    \node[draw,very thick,C9, fill=C9!50, trapezium, trapezium angle=75, shape border rotate=270, minimum
        width=1cm, minimum height=0.5cm, right=0.8cm of twoonehots, label=below:$f_\phi$] (fphi) {};
        
    \node[right=0.8cm of fphi, yshift = 0cm] (phis) {$\phi_1$ $\phi_3$};

    \node[right=0.8cm of phis, yshift = -0.25cm] (qphi1) {$z_3 \sim q_{\phi_3}(z_3 \vert x)$};
    \node[right=0.8cm of phis, yshift =0.25cm] (qphi3) {$z_1 \sim q_{\phi_1}(z_1 \vert x)$};

     \node[draw,very thick,C1, fill=C1!50, trapezium, trapezium angle=75, shape border rotate=90, minimum
        width=1cm, minimum height=0.5cm, right=5.5cm of fphi, label=below:Decoder] (decoder) {};
    \node[right=0.8cm of decoder, yshift =0cm] (theta) {$\theta$};
        
    \node[right=0.8cm of theta, yshift =0cm] (elbo) {$\widetilde{\mathcal{L}}_\text{\modelname{S2S}}$};
    
    \draw[-latex] (data) to ([yshift=0.0cm]fd2h.west);
    \draw[-latex] (fd2h) to ([yshift=0.0cm]hidden.west);

    \draw[-latex] (hidden) -- ++(0.5,0) |- (twoonehots);
      \draw[-latex] (phis) -- ++(1,0) -- ++(0,-1)-- ++(6.375,0) -- (elbo);
      
        \draw[-latex] ([xshift=-0.5cm]qphi3.north) -- ++(0,0.4) -- ++(0,0.4) -- ++(5.4,0) -- (elbo);

    \draw[-latex] (theta) to (elbo);
    \draw[-latex] (onehot) to (twoonehots);
    \draw[-latex] (twoonehots) to (fphi);
    \draw[-latex] (fphi) to (phis);
    \draw[-latex] (phis) to ([yshift=0.25cm]qphi1.west);
    \draw[-latex] ([yshift=0.25cm]qphi1.east) to (decoder);
    %\draw[-latex] (qphi3) to (decoder);
    \draw[-latex] (decoder) to (theta);

\end{tikzpicture}
\caption{Block diagram depicting the estimation of MISELBO using \textbf{\modelname{MISVAE}} with the \textbf{$\mathrm{S2S}$} estimator, with $S=2$ and $A=3$.  First, \(f_{D2\mathcal{H}}\) maps the data to an intermediate hidden space, producing a representation \(h\). The next network, \(f_\phi\), takes \(h\) along with \(S\) \(A\)-dimensional one-hot encodings, acting as signals of the \(S\) mixtures used by the \(\text{S2S}\) estimator, as input, which are then mapped to the variational parameters, here \(\phi_1\) and \(\phi_2\), of the mixture components. Samples drawn from the \(S\) mixtures are then passed to a decoding network to produce the parameters \(\theta\) of the generative model. Collectively, the sampled latent variables, the variational parameters, and \(\theta\) , are used to compute \(\widetilde{\mathcal{L}}_{\text{S2S}}\). The diagram is explained in detail in Sec. \ref{sec:misvae}. Corresponding diagrams for the $\mathrm{S2A}$ and $\mathrm{A2A}$ estimators can be found in Fig. \ref{fig:misvae-architecture-s2a-a2a}.}
\label{fig:misvae-architecture}
\end{figure*}

\paragraph{Summary}
Our two new estimators have lower computational complexity than the \modelname{A2A} estimator. As we will demonstrate, the benefits gained in practice in terms of runtime will be particularly important if the numerator (the generative model) of Eq. \eqref{eq:miselbo} is expensive to compute. Comparing \modelname{S2A} and \modelname{S2S}, the latter will enjoy a shorter inference time if the entropy, or, alternatively, the denominator in Eq. \eqref{eq:miselbo}, is expensive to compute.

\section{Multiple Importance Sampling VAE}
\label{sec:misvae}
Impressive results have been obtained by naively expanding the parameter space with the number of mixtures. However, by carefully studying the problem at hand, similar, or even improved, performance gains can be achieved at negligible increases in parameter costs.

\modelname{MISVAE} is a new Mixture VAE architecture featuring an encoder network composed of two consecutive networks. The novelty of \modelname{MISVAE} lies in the second network, which parameterizes the mixture components using amortization.

The first network maps the data to an intermediate (deterministic) hidden space, $\mathcal{H}$; we refer to it as the D2$\mathcal{H}$ net and denote the function as $f_{D2\mathcal{H}}$. The second net is a mapping from the Cartesian product of $\mathcal{H}$ and the space of $A$-dimensional one-hot encodings, $\mathcal{O}_A$, to the parameters of a mixture component. We denote this as $f_\phi$ and call it the amortized mixture parameterization (AMP) net. We write\begin{equation}
    f_{D2\mathcal{H}}: \mathcal{X}\mapsto\mathcal{H},~~f_{\phi}: \mathcal{H}\times \mathcal{O}_A \mapsto \mathcal{Q},
\end{equation}
or, alternatively, $h = f_{D2\mathcal{H}}(x)$ and if $\mathcal{Q}$ is the family of Gaussians, $(\mu(h, {\phi_s}), \sigma(h, {\phi_s}))=f_ \phi(h, o_A(s))$, where $o_A(s)$ is an $A$ long one-hot encoding with the $s$-th element set to one. 

The AMP net, $f_\phi$, is a sequence of neural networks (NNs), shared among all mixture components. The NNs take $o_A(s)$ as their biases. That is, to get the parameters of the $s$-th component, we pass $o_A(s)$ as a bias to the NNs. See Fig. \ref{fig:misvae-architecture}  for a depiction of the \modelname{MISVAE} architecture.

\section{Experiments}
\label{sec:experiments}
In this section, we infer variational parameters using \modelname{MISVAE} along with the \modelname{S2S}, \modelname{S2A}, and \modelname{A2A} estimators. We conduct comparisons among these methods and against SOTA approaches across a synthetic dataset, three image datasets, and eight phylogenetic datasets. All code necessary to replicate our experiments is publicly available at: \url{https://github.com/okviman/efficient-mixtures}.

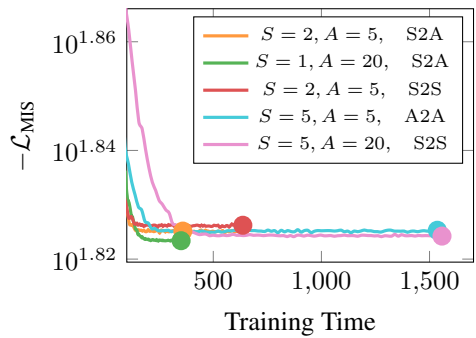
\begin{figure}[!htbp]
    \centering 
    \centering
    \begin{tikzpicture}
\begin{axis}[
    xlabel={Training Time},
    ylabel={$-\mathcal{L}_\textnormal{MIS}$},
    legend pos=north east,
    legend style={font=\scriptsize},
    ymode = log,
    xmin=100,
    width = 0.75\linewidth, height = 0.6\linewidth
]

\addplot[smooth,C1!80, very thick, each nth point=5,] table [col sep=comma, x=Time, y expr=-\thisrow{ELBO} ] {data_S2_A5_s2a.csv};
\addlegendentry{$S=2, A=5, \quad\mathrm{S2A}$}
\addplot[fill=C1!80, very thick, forget plot, mark size=3pt, mark=*, mark options={fill=C1!80}, C1!80] coordinates {(359.2495627403239, 66.923676)};

\addplot[smooth, C2!80, very thick, each nth point=5,] table [col sep=comma, x=Time, y expr=-\thisrow{ELBO}] {data_S1_A20_s2a.csv};
\addlegendentry{$S=1, A=20,\quad \mathrm{S2A}$}
\addplot[fill=C2!80, very thick, mark=*,mark size=3pt, forget plot, mark options={fill=C2!80}, C2!80] coordinates {(351.202199697494, 66.6476)};

\addplot[smooth,C3!80, very thick, each nth point=5,] table [col sep=comma, x=Time, y expr=-\thisrow{ELBO}] {data_S2_A5_s2s.csv};
\addlegendentry{$S=2, A=5,\quad \mathrm{S2S}$}
\addplot[fill=C3!80, very thick, mark=*,mark size=3pt, forget plot, mark options={fill=C3!80}, C3!80] coordinates {(636.7559792995494, 67.08322)};

\addplot[smooth, C9!80, very thick, each nth point=5,] table [col sep=comma, x=Time, y expr=-\thisrow{ELBO}] {data_S5_A5_a2a.csv};
\addlegendentry{$S=5, A=5, \quad\mathrm{A2A}$}
\addplot[fill=C9!80, very thick, mark=*,mark size=3pt,  forget plot,mark options={fill=C9!80}, C9!80] coordinates {(1536.1006073951728, 66.94394)};

\addplot[smooth, C6!80, very thick, each nth point=5,] table [col sep=comma, x=Time, y expr=-\thisrow{ELBO}] {data_S5_A20_s2s.csv};
\addlegendentry{$S=5, A=20, \quad\mathrm{S2S}$}
\addplot[fill=C6!80, very thick, mark=*,mark size=3pt, forget plot, mark options={fill=C6!80}, C6!80] coordinates {(1557.6487603187738, 66.78039)};
\end{axis}
\end{tikzpicture}

    
    \caption{Comparison of MISELBO approximation performance and training runtimes across three distinct estimators under various settings of \(S\) and \(A\) in the Toy Experiment, trained for $50,000$ epochs.}
    \label{fig:ar_toy_posterior_curves}
    \vspace{-5pt}
\end{figure}

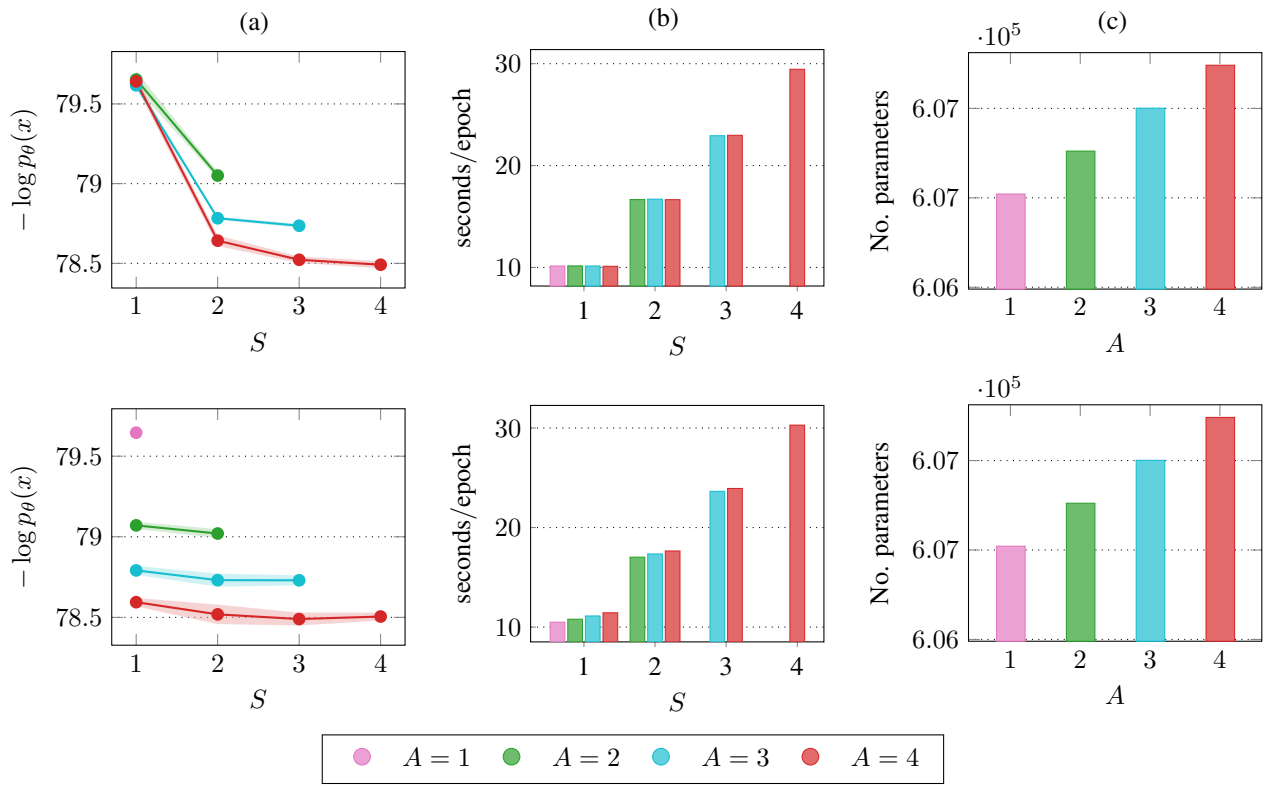
\begin{figure*}[ht]
\centering
\setlength{\abovecaptionskip}{10pt}
\begin{minipage}{0.32\textwidth}
\centering
  \hspace{35pt} (a)\\
  \vspace{5pt}
  \centering
  \begin{tikzpicture}
    \begin{axis}[
        xlabel={$S$},
         ylabel={$  -\log p_{\theta}(x)$},
        legend pos=north east,
        major grid style={dotted,black},
        ymajorgrids=true, 
        width=0.999\linewidth,
        xtick={1,2,3,4},
        legend style={at={(0.5,-0.3)},
        anchor=north,legend columns=2, font=\small},
    ]
    \addplot[color=\Sone, mark=*, mark options={solid}, thick, name path=mean1]
    coordinates {
      (1,79.62780311)
    };
    \addplot[name path=upper1, draw=none] coordinates {
      (1,79.62780311 + 0.02541483835)
    };
    \addplot[name path=lower1, draw=none] coordinates {
      (1,79.62780311 - 0.02541483835)
    };
    \addplot[\Sone!40, fill opacity=0.5] fill between[of=upper1 and lower1];
    \addplot[color=\Stwo, mark=*, mark options={solid}, thick, name path=mean2]
    coordinates {
        (1,79.65321958)
        (2,79.05084711)
    };
    \addplot[name path=upper2, draw=none] coordinates {
      (1,79.65321958 + 0.05165358563)
      (2,79.05084711 + 0.02124002654)
    };
    \addplot[name path=lower2, draw=none] coordinates {
      (1,79.65321958 - 0.05165358563)
      (2,79.05084711 - 0.02124002654)
    };
    \addplot[\Stwo!40, fill opacity=0.5] fill between[of=upper2 and lower2];
    \addplot[color=\Sthree, mark=*, mark options={solid}, thick, name path=mean3]
    coordinates {
        (1,79.61649998)
        (2, 78.78297256)
        (3,78.73591976)
    };
    \addplot[name path=upper3, draw=none] coordinates {
      (1,79.61649998 + 0.01939768964)
      (2,78.78297256 + 0.00966533411)
      (3,78.73591976 + 0.01243831165)
    };
    \addplot[name path=lower3, draw=none] coordinates {
      (1,79.61649998 - 0.01939768964)
      (2,78.78297256 - 0.00966533411)
      (3,78.73591976 - 0.01243831165)
    };
    \addplot[\Sthree!40, fill opacity=0.5] fill between[of=upper3 and lower3];
    \addplot[color=\Sfour, mark=*, mark options={solid}, thick, name path=mean4]
    coordinates {
        (1,79.64138781)
        (2,78.64244235)
        (3,78.52243277)
        (4,78.49139702)
    };
    \addplot[name path=upper4, draw=none] coordinates {
        (1,79.64138781 + 0.0329175222)
        (2,78.64244235 + 0.03340940546)
        (3,78.52243277 + 0.02004899853)
        (4,78.49139702 + 0.02236409292)
    };
    \addplot[name path=lower4, draw=none] coordinates {
        (1,79.64138781 - 0.0329175222)
        (2,78.64244235 - 0.03340940546)
        (3,78.52243277 - 0.02004899853)
        (4,78.49139702 - 0.02236409292)
    };
    \addplot[\Sfour!40, fill opacity=0.5] fill between[of=upper4 and lower4];

    \end{axis}
  \end{tikzpicture}
\end{minipage}
\hspace{0.0025\textwidth}
\begin{minipage}{0.32\textwidth}
  \centering
  \hspace{20pt} (b)\\
  \vspace{5pt}

\begin{tikzpicture}
\begin{axis}[
    xmin=0.25,
    xmax=4.375,
    ybar,
    width=0.999\linewidth,
    ylabel={$\text{seconds}/\text{epoch}$},
    bar width=0.925\mybarwidth,
    xtick={1,2,3,4},
    xticklabels={1,2,3,4},
    major grid style={dotted,black},
        legend style={at={(0.5,-0.3)},
        anchor=north,legend columns=2, font=\small},
    ymajorgrids=true,
    xlabel={$S$},
        tick style={
        tick align=inside,
        tick pos=left,
        major tick length=2pt
    },
    /pgf/bar shift auto/.style={
        /pgf/bar width/.initial=1pt,
    },
]
\addplot[fill=\Sone!69, draw=\Sone, bar shift=-1.8\mybarwidth*0.925] coordinates {(1,10.14)};
\addplot[fill=\Stwo!69, draw=\Stwo, bar shift=-0.6\mybarwidth*0.925] coordinates {(1,10.15)};
\addplot[fill=\Sthree!69, draw=\Sthree, bar shift=0.6\mybarwidth*0.925] coordinates {(1,10.14) };
\addplot[fill=\Sfour!69, draw=\Sfour, bar shift=1.8\mybarwidth*0.925] coordinates {(1,10.11)};

\addplot[fill=\Stwo!69, draw=\Stwo, bar shift=-1.2\mybarwidth*0.925] coordinates {(2,16.66)};
\addplot[fill=\Sthree!69, draw=\Sthree, bar shift=0\mybarwidth*0.925] coordinates { (2,16.7)};
\addplot[fill=\Sfour!69, draw=\Sfour, bar shift=1.2\mybarwidth*0.925] coordinates {(2,16.65)};

\addplot[fill=\Sthree!69, draw=\Sthree, bar shift=-0.6\mybarwidth*0.925] coordinates {(3,22.93)};
\addplot[fill=\Sfour!69, draw=\Sfour, bar shift=0.6\mybarwidth*0.925] coordinates { (3,22.97)};

\addplot[fill=\Sfour!69, draw=\Sfour, bar shift=0\mybarwidth*0.925] coordinates {(4,29.45)};

\end{axis}
\end{tikzpicture}

\end{minipage}
\hspace{0.0025\textwidth}
\begin{minipage}{0.32\textwidth}
\centering
  \hspace{35pt} (c)\\
  \vspace{-5pt}
  \centering

\begin{tikzpicture}
\begin{axis}[
    ybar ,
    enlargelimits=0.025,
    width=0.999\linewidth,
    ylabel={$ \text{No. parameters}$},
     bar width=1.8\mybarwidth,
    xtick={1,2,3,4},
    legend style={at={(0.5,-0.4)},
      anchor=north,legend columns=2, font=\small, opacity=0},
    ymajorgrids=true,
    xlabel={$A$},
    major grid style={dotted,black},
    xmax = 4.5,
    xmin=0.5,
    ymax = 607780,
    ymin= 606521,
    xtick align=inside,
    bar shift=0.025\mybarwidth
    ]
    
\addplot[fill=\Sone!69, draw=\Sone] coordinates {(1, 607021)};
\addplot[fill=\Stwo!69, draw=\Stwo] coordinates{(2, 607261)};
\addplot[fill=\Sthree!69, draw=\Sthree] coordinates {(3, 607501) };
\addplot[fill=\Sfour!69, draw=\Sfour] coordinates {(4, 607741)};
\end{axis}
\end{tikzpicture}
\end{minipage}
\begin{minipage}{0.32\textwidth}
\vspace{-20pt}
  \centering
  \begin{tikzpicture}
    \begin{axis}[
        xlabel={$S$},
         ylabel={$  -\log p_{\theta}(x)$},
        legend pos=north east,
        major grid style={dotted,black},
        ymajorgrids=true,
        width=0.999\linewidth,
        xtick={1,2,3,4},
        legend style={at={(0.5,-0.3)},
        anchor=north,legend columns=2, font=\small},
    ]
    \addplot[color=\Sone, mark=*, mark options={solid}, thick, name path=mean1]
    coordinates {
      (1,79.64566655)
    };
    \addplot[name path=upper1, draw=none] coordinates {
      (1,79.64566655 + 0.02634546997)
    };
    \addplot[name path=lower1, draw=none] coordinates {
      (1,79.64566655 - 0.02634546997)
    };
    \addplot[\Sone!40, fill opacity=0.5] fill between[of=upper1 and lower1];

    \addplot[color=\Stwo, mark=*, mark options={solid}, thick, name path=mean2]
    coordinates {
      (1,79.07078828)
      (2,79.02022165)
    };
    \addplot[name path=upper2, draw=none] coordinates {
      (1,79.07078828 + 0.02177742407)
      (2,79.02022165 + 0.02743518316)
    };
    \addplot[name path=lower2, draw=none] coordinates {
      (1,79.07078828 - 0.02177742407)
      (2,79.02022165 - 0.02743518316)
    };
    \addplot[\Stwo!40, fill opacity=0.5] fill between[of=upper2 and lower2];
    \addplot[color=\Sthree, mark=*, mark options={solid}, thick, name path=mean3]
    coordinates {
      (1,78.79146045)
      (2,78.73095773)
      (3,78.72992813)
    };
    \addplot[name path=upper3, draw=none] coordinates {
      (1,78.79146045 + 0.02775104247)
      (2,78.73095773 + 0.0406363633)
      (3,78.72992813 + 0.031)
    };
    \addplot[name path=lower3, draw=none] coordinates {
      (1,78.79146045 - 0.02775104247)
      (2,78.73095773 - 0.0406363633)
      (3,78.72992813 - 0.03136226658)
    };
    \addplot[\Sthree!40, fill opacity=0.5] fill between[of=upper3 and lower3];

    \addplot[color=\Sfour, mark=*, mark options={solid}, thick, name path=mean4]
    coordinates {
      (1,78.59430804)
      (2,78.51864284)
      (3,78.48964093)
      (4,78.50517191)
    };
    \addplot[name path=upper4, draw=none] coordinates {
      (1,78.59430804 + 0.02647001335)
      (2,78.51864284 + 0.06012384635)
      (3,78.48964093 + 0.04005568917)
      (4,78.50517191 + 0.02455043531)
    };
    \addplot[name path=lower4, draw=none] coordinates {
      (1,78.59430804 - 0.02647001335)
      (2,78.51864284 - 0.06012384635)
      (3,78.48964093 - 0.04005568917)
      (4,78.50517191 - 0.02455043531)
    };
    \addplot[\Sfour!40, fill opacity=0.5] fill between[of=upper4 and lower4];

    \end{axis}
  \end{tikzpicture}
\end{minipage}
\hspace{0.0025\textwidth}
\begin{minipage}{0.32\textwidth}
\vspace{-20pt}
  \centering

\begin{tikzpicture}
\begin{axis}[
    xmin=0.25,
    xmax=4.375,
    ybar,
    width=0.999\linewidth,
    ylabel={$\text{seconds}/\text{epoch}$},
    bar width=0.925\mybarwidth, 
    xtick={1,2,3,4},
    xticklabels={1,2,3,4},
    major grid style={dotted,black},
        legend style={at={(0.5,-0.3)},
        anchor=north,legend columns=2, font=\small},
    ymajorgrids=true,
    xlabel={$S$},
        tick style={
        tick align=inside,
        tick pos=left,
        major tick length=2pt
    },
    /pgf/bar shift auto/.style={
        /pgf/bar width/.initial=1pt,
    },
]
\addplot[fill=\Sone!69, draw=\Sone, bar shift=-1.8\mybarwidth*0.925] coordinates {(1,10.4931433)};
\addplot[fill=\Stwo!69, draw=\Stwo, bar shift=-0.6\mybarwidth*0.925] coordinates {(1,10.79079)};
\addplot[fill=\Sthree!69, draw=\Sthree, bar shift=0.6\mybarwidth*0.925] coordinates {(1,11.11618) };
\addplot[fill=\Sfour!69, draw=\Sfour, bar shift=1.8\mybarwidth*0.925] coordinates {(1,11.43636)};
\addplot[fill=\Stwo!69, draw=\Stwo, bar shift=-1.2\mybarwidth*0.925] coordinates {(2,17.025889)};
\addplot[fill=\Sthree!69, draw=\Sthree, bar shift=0\mybarwidth*0.925] coordinates { (2,17.3465)};
\addplot[fill=\Sfour!69, draw=\Sfour, bar shift=1.2\mybarwidth*0.925] coordinates {(2,17.646281)};
\addplot[fill=\Sthree!69, draw=\Sthree, bar shift=-0.6\mybarwidth*0.925] coordinates {(3,23.640036)};
\addplot[fill=\Sfour!69, draw=\Sfour, bar shift=0.6\mybarwidth*0.925] coordinates { (3,23.928404)};
\addplot[fill=\Sfour!69, draw=\Sfour, bar shift=0\mybarwidth*0.925] coordinates {(4,30.28939)};
\end{axis}
\end{tikzpicture}

\end{minipage}
\hspace{0.0025\textwidth} 
\begin{minipage}{0.32\textwidth}
\vspace{-20pt}
  \centering
\vspace{0.15\textwidth}
\begin{tikzpicture}
\begin{axis}[
    ybar ,
    enlargelimits=0.025,
    width=0.999\linewidth,
    ylabel={$ \text{No. parameters}$},
     bar width=1.8\mybarwidth,
    xtick={1,2,3,4},
    legend style={at={(0.5,-0.4)},
      anchor=north,legend columns=2, font=\small, opacity=0},
    ymajorgrids=true,
    xlabel={$A$},
    major grid style={dotted,black},
    xmax = 4.5,
    xmin=0.5,
    ymax = 607780,
    ymin= 606521,
    xtick align=inside, 
    bar shift=0.025\mybarwidth
    ]
\addplot[fill=\Sone!69, draw=\Sone] coordinates {(1, 607021)};
\addplot[fill=\Stwo!69, draw=\Stwo] coordinates{(2, 607261)};
\addplot[fill=\Sthree!69, draw=\Sthree] coordinates {(3, 607501) };
\addplot[fill=\Sfour!69, draw=\Sfour] coordinates {(4, 607741)};

 \legend{$S=1$,$S=2$, $S=3$,$S=4$}
\end{axis}
\end{tikzpicture}
\end{minipage}
\vspace{-30pt}
\fbox{
\begin{tabular}{cccccccc}
\tikz\draw[\Sone,fill=\Sone!69] (0,0) circle (.7ex); & $A=1$ &
\tikz\draw[\Stwo,fill=\Stwo!69] (0,0) circle (.7ex); & $A=2$ &
\tikz\draw[\Sthree,fill=\Sthree!69] (0,0) circle (.7ex); & $A=3$ &
\tikz\draw[\Sfour,fill=\Sfour!69] (0,0) circle (.7ex); & $A=4$ \\
\end{tabular}}
\caption{ Results on \textbf{MNIST} for \modelname{MISVAE}  trained with various combinations of $S$ and $A$, with the $\mathrm{S2S}$ estimator (top row) and the $\mathrm{S2A}$ estimator (bottom row). (a) Average (solid) NLL results computed over three runs with one standard deviation (opaque) displayed, (b) training time per epoch, and (c) the number of network parameters for \modelname{MISVAE} for increasing values of $A$. Using \(\text{MISVAE}\), the number of network parameters increases by a small amount as we increase \(A\). Also, with the \(\text{S2S}\) estimator, we can keep \(S\) fixed and increase \(A\), without impacting the number of seconds needed to complete an epoch and simultaneously improving the NLL. For \(\text{S2A}\), we converge to an equivalent solution with \(A\) held fixed for any \(S < A\), meaning that in practice, we can scale up \(A\) for small values of \(S\) at a small extra computational cost per mixture component.}
\label{fig:mnist_small_S_and_A_s2s_vs_sva}
\end{figure*}

\subsection{Toy Example}
Here we construct a toy example in order to evaluate the performances of the estimators when the posterior exhibits certain properties. Concretely, we design a generative model that is non-Gaussian, has an autoregressive likelihood function, and assume that the $A$ terms in the energy term, $\frac{1}{A}\sum_{a=1}^A\mathbb{E}_{q_{\phi_a}(z_a|x)}[\log p_\theta(z_a, x)]$, have to be computed sequentially.

These criteria are of interest as they naturally arise in many settings, including all our real-data experiments below: the posteriors are non-Gaussian, both the pixel-CNN decoder and the likelihood function in Bayesian phylogenetics are autoregressive, and the corresponding energy terms are not always parallelizable–––the CIFAR-10 images are too big to parallelize over $A$ with a reasonable batch size, and the standard implementation of the dynamic programming required to compute the phylogenetic likelihood function does not necessarily parallelize.

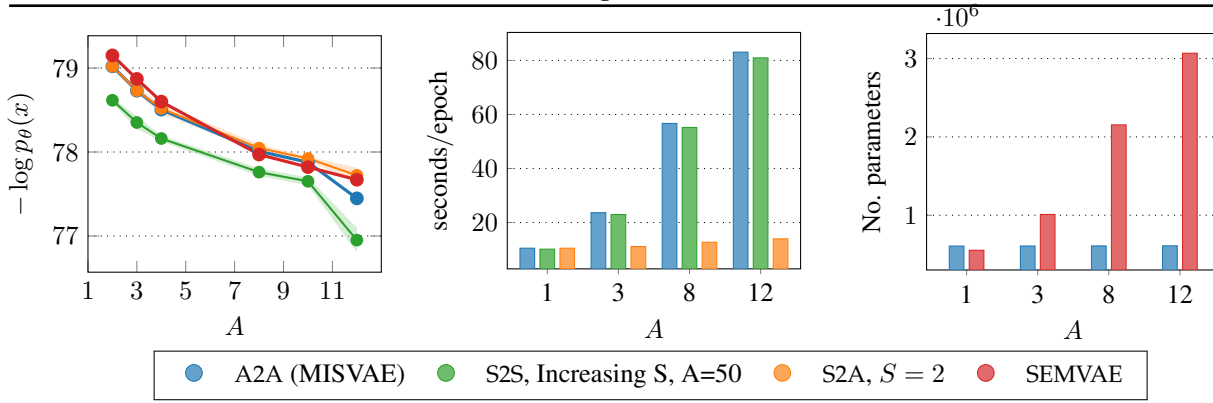
\begin{figure*}[t]

\centering
\begin{minipage}{0.32\textwidth}
\centering
  \begin{tikzpicture}[trim axis left, trim axis right]
    \begin{axis}[
        xlabel={$A$},
         ylabel={$  -\log p_{\theta}(x)$},
        legend pos=north east,
        major grid style={dotted,black},
        ymajorgrids=true,
        width=0.999\linewidth,
        xtick={1,3,5,7, 9, 11},
        legend style={at={(0.5,-0.3)},
        anchor=north,legend columns=2, font=\small},
    ]
  \addplot[
        color=\Sfifty,
        mark=*,
        mark options={solid},
        very thick
        ]
        coordinates {
    (2,79.02022165)
    (3,78.72992813)
    (4,78.50517191)
    (8, 78.00826563)
    (10, 77.87825156)
    (12, 77.44778281)
        };

    \addplot[color=\Stwo, mark=*, mark options={solid}, thick, name path=s2sbigA]
    coordinates {
        (2,78.61600947)
        (3, 78.35156455)
        (4,78.1608179)
	(8,77.76003351)
         (10, 77.65113906)
	(12,76.95066702)
    };
    \addplot[name path=uppers2sbigA, draw=none] coordinates {
        (2,78.61600947 +0.03142932523)
        (3, 78.35156455+0.07533286365)
        (4,78.1608179+0.03542069777)
	(8,77.76003351+0.04479344381)
         (10, 77.65113906+0.048)
	(12,76.95066702+0.1474571023)
    };
    \addplot[name path=lowers2sbigA, draw=none] coordinates {
        (2,78.61600947 -0.03142932523)
        (3, 78.35156455-0.07533286365)
        (4,78.1608179-0.03542069777)
	(8,77.76003351-0.04479344381)
         (10, 77.65113906-0.048)
	(12,76.95066702-0.1474571023)
    };
    \addplot[\Stwo!40, fill opacity=0.5] fill between[of=uppers2sbigA and lowers2sbigA];
    \addplot[color=C1, mark=*, mark options={solid}, thick, name path=misvaes2an2]
    coordinates {
        (2,79.02022165)
        (3, 78.73095773)
        (4,78.51864284)
	(8,78.04645052)
         (10, 77.92294207)
	(12,77.72173854)
    };
    \addplot[name path=uppermisvaes2an2, draw=none] coordinates {
      (2,79.02022165+ 0.02743518316)
	(3, 78.73095773+0.0406363633)
	(4,78.51864284+ 0.06012384635)
	(8,78.04645052+0.04556460325)
         (10, 77.92294207+ 0.02245775397)
	(12,77.72173854+0.09217452346)
    };
    \addplot[name path=lowermisvaes2an2, draw=none] coordinates {
      (2,79.02022165- 0.02743518316)
	(3, 78.73095773-0.0406363633)
	(4,78.51864284- 0.06012384635)
	(8,78.04645052-0.04556460325)
        (10, 77.92294207- 0.02245775397)
	(12,77.72173854-0.09217452346)
    };
    \addplot[C1!40, fill opacity=0.5] fill between[of=uppermisvaes2an2 and lowermisvaes2an2];
           \addplot[
        color=\Sfour,
        mark=*,
        mark options={solid},
        very thick
        ]
        coordinates {
        (2,79.15)
        (3,78.87)
        (4,78.6)
        (8,77.97)
        (10,77.82)
        (12,77.67)
        };
    
    \end{axis}
  \end{tikzpicture}
\end{minipage}
\begin{minipage}{0.32\textwidth}
  \centering

\begin{tikzpicture}[trim axis left, trim axis right]
\begin{axis}[
    ybar,
    width=0.999\linewidth,
    ylabel={$\text{seconds}/\text{epoch}$},
    bar width=0.925\mybarwidth,
    symbolic x coords={1,3,8,12},
    enlarge x limits= 0.1875, 
    xtick={1,3,8, 12},
    major grid style={dotted,black},
        legend style={at={(0.5,-0.3)},
        anchor=north,legend columns=2, font=\small},
    ymajorgrids=true,
    xlabel={$A$},
        tick style={
        tick align=inside,
        tick pos=left,
        major tick length=2pt
    },
    /pgf/bar shift auto/.style={
        /pgf/bar width/.initial=1pt,
    },
]
\addplot[
    fill=\Sfifty!69,
    draw=\Sfifty
] coordinates {
    (1,10.49314336)
    (3,23.64003682)
    (8,56.68440416)
    (12,83.07941034)
};

\addplot+[
fill=\Stwo!69, draw=\Stwo] coordinates {
    (1,10.135)
    (3,22.955)
    (8,55.20333333)
    (12,80.95666667)
};

\addplot+[
fill=C1!69, draw=C1] coordinates {
    (1,10.49314336)
    (3,11.11618503)
    (8,12.6939984)
    (12,13.95162986)
};
\end{axis}
\end{tikzpicture}
\end{minipage}
\begin{minipage}{0.32\textwidth}
  \centering
\vspace{-10pt}
\begin{tikzpicture}[trim axis left, trim axis right]
\begin{axis}[
    ybar,
    width=0.999\linewidth, 
    ylabel={No. parameters},
    bar width=0.925\mybarwidth,
    symbolic x coords={1,3,8,12},
    enlarge x limits= 0.1875, 
    xtick={1,3,8, 12},
    major grid style={dotted,black},
        legend style={at={(0.5,-0.3)},
        anchor=north,legend columns=2, font=\small},
    ymajorgrids=true,
    xlabel={$A$},
        tick style={
        tick align=inside,
        tick pos=left,
        major tick length=2pt
    },
    /pgf/bar shift auto/.style={
        /pgf/bar width/.initial=1pt,
    },
]
\addplot[
    fill=\Sfifty!69,
    draw=\Sfifty
] coordinates {
    (1,607021)
    (3,607501)
    (8,608701)
    (12,609661)
};

\addplot+[
fill=\Sfour!69, draw=\Sfour] coordinates {
    (1,553181)
    (3,1010389)
    (8,2153409)
    (12,3067825)
};
\end{axis}
\end{tikzpicture}
\end{minipage}
\centering
\fbox{
\begin{tabular}{cccccccccc}
\tikz\draw[\Sfifty,fill=\Sfifty!69] (0,0) circle (.7ex); & \modelname{A2A} (MISVAE)&
\tikz\draw[\Stwo,fill=\Stwo!69] (0,0) circle (.7ex); & \modelname{S2S}, Increasing S, A=50& 
\tikz\draw[C1,fill=C1!69] (0,0) circle (.7ex); &  \modelname{S2A}, $S=2$ &
\tikz\draw[\Sfour,fill=\Sfour!69] (0,0) circle (.7ex); & \modelname{SEMVAE} \\
\end{tabular}}
\caption{Comparison between SEMVAE and \modelname{MISVAE} using the S2S, A2A, and S2A estimators on \textbf{MNIST}: (a) NLL scores for increasing values of \(A\), (b) training time per epoch, and (c) number of hyperparameters for increasing \(A\) for SEMVAE compared to \modelname{MISVAE}. Note: The green curve represents the performance of \modelname{MISVAE} using the S2S estimator with \(S\) increasing , such that S=A on the x-axis, while \(A\) is held fixed at \(50\).}\label{fig:misvae_semvae_estimator_comparison}

\end{figure*}

\begin{table*}[b]
  \caption{NLL statistics for SOTA VAE architectures on \textbf{MNIST}. The Composite model is a SEMVAE model with hierarchical models, NFs and the VampPrior. For IWAE, $L$ is the number of importance samples used during training.}
  \label{tab:NLL SOTA}
  \centering
  \begin{tabular}{lccc}
    \toprule

    Model     & NLL  & No. Parameters & Seconds/epoch \\
    \midrule
    IWAE ($L =20$) \citep{burda2016importance} & $79.63$ & $720,541$ &$128.02$\\
    
    Hierarchical VAE w. VampPrior \citep{tomczak2018vae} & $78.45$ & $1,777,821$ &-\\
    NVAE \citep{vahdat2020nvae} & $78.01$ & $33,363,134$&-\\
    
    MAE \citep{ma2019mae} & $77.98$ & $1,565,570$\\
    Ensemble NVAE \citep{kviman2022multiple} & $77.77$ $\pm$ $0.2$& -&-\\
    Vanilla SEMVAE ($S=12$; \citet{kviman2023cooperation}) & $77.67$ & $8,549,344$&-\\
    
    Composite SEMVAE ($S=4$; \citet{kviman2023cooperation})     & $77.23$ $\pm$ 0.1 & $5,212,065$ &-\\
    CR-NVAE \citep{NEURIPS2021_6c19e0a6}     & $76.93$ & - &-\\
    \midrule
    \modelname{MISVAE} $\mathrm{S2S}$ ($A=50, S=20$; \textbf{our})     & $76.67$   &$\boldsymbol{618,781}$& $132.46$\\
    \modelname{MISVAE} $\mathrm{S2A}$ ($A=200, S=1$; \textbf{our})     & $75.43$ & $654,781$&$ \textbf{73.06}$\\
    \modelname{MISVAE} $\mathrm{S2A}$ ($A=800, S=1$; \textbf{our})     & $\boldsymbol{74.07}$ & $798,781$&$261.71$\\
    \bottomrule
  \end{tabular}
\end{table*}

With these criteria in mind, we let $p(x|z) = \prod_{n=1}^N p(x_n|z)$, where $x_n$ is the $n$-th $d_x$-dimensional data point, $N$ is the number of generated data points and $p(x_n|z) = \prod_{i=1}^{d_x}\text{Bernoulli}\left(x^{(i)}\Big|\text{sigmoid}\left(\theta^{(i)} + \sum_{j=1}^{i-1} \beta^{i-j} x^{(j)}\right)\right)$, with $\beta=0.1$, superindices are within parentheses, and $\theta = Wz$, where \begin{equation}
 W\in \mathbb{R}^{d_x\times d_z}, \quad  W_{u,v}\sim \log\mathcal{N}(0, 0.1).
\end{equation}
Finally, as prior we use the hierarchical \textit{Neal's funnel} model
\begin{equation}
    p(z_2, z_1) = \mathcal{N}(z_2|0, e^{z_1 / 2})\mathcal{N}(z_1|0, 3).
\end{equation}
We constrain our analysis to a 2-dimensional latent space for visualization purposes. An unnormlized posterior when $d_x=20$ and $N=5$ is depicted in Fig. \ref{fig:ar_toy_posterior}.

\textbf{Results.} The posterior is non-Gaussian (and intractable) and has an autoregressive likelihood function. We artificially force the energy term to be computed sequentially, let $d_x=20$ (a moderately large number) and chose $N=5$. In Fig. \ref{fig:ar_toy_posterior} we visualize the unnormalized posterior when $d_x=20$. All mixture components in the variational approximations are Gaussians with diagonal covariance matrices.

The purpose of this experiment is to compare the MISELBO values and total runtimes across different estimators, as well as to visualize the final variational approximations in the latent space. Accordingly, all models were subjected to training over $50,000$ epochs. Fig. \ref{fig:ar_toy_posterior_curves} presents the evolution of MISELBO during the training process for these estimators, alongside the total runtime required to finish the epochs. If $d_x$ was to be further increased, so would the cost of evaluating the likelihood. Hence, we expect our new estimators to provide good approximations in shorter runtime than the $\mathrm{A2A}$ estimator when $d_x$ is sufficiently high.

The results in Fig. \ref{fig:ar_toy_posterior_curves} confirm our expectations. When $(S=2, A=5)$, the S2A and A2A estimators achieve the same MISELBO scores, however, S2A requires only a fraction of the runtime. Meanwhile, for the said $S$ and $A$, S2S performs worse in terms of MISELBO score, albeit in short runtime. Interestingly, as S2S can be regarded as an ensemble, S2S with $(S=5, A=20)$ can outperform A2A ($A=5$) in terms of MISELBO scores in approximately the same training time per epoch. Finally, as the S2A is an unbiased estimator of MISELBO for any $S<A$, it excels when utilizing a large number of mixtures $A$ with a small $S$. Notably, the configuration of S2A with ($S=1,A=20$) not only boasts the fastest training time per epoch but also achieves the lowest overall negative MISELBO scores.

In Appendix \ref{app:toy_example} we include further implementation details and visualizations of the approximations in the latent space.

\subsection{Image Data}

To make our results comparable to current SOTA mixture architectures, we use the  same experiment setup and training-related hyperparameters as in \citet{kviman2023cooperation}. We refer to the benchmark \modelname{MISVAE} with the Mixture VAE used in \citet{kviman2023cooperation}, where separate encoder networks are used for each mixture component, as the separate encoder Mixture VAE (SEMVAE). We train on MNIST \cite{lecun-mnisthandwrittendigit-2010}, FashionMNIST \cite{DBLP:journals/corr/abs-1708-07747}, and CIFAR-10 \cite{cifar}. Additional details are provided in Appendix~\ref{app:exp_det}

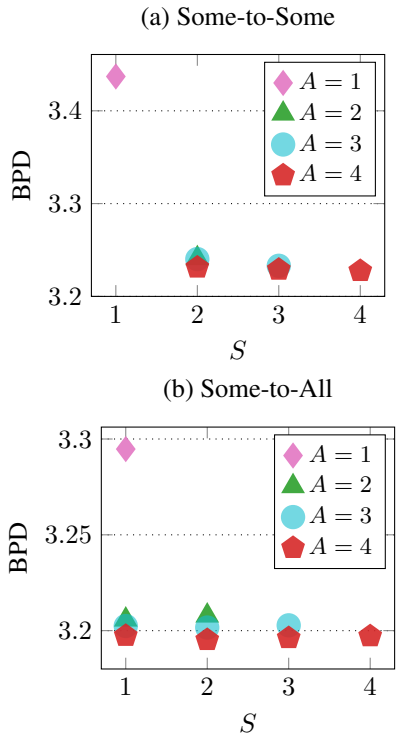
\begin{figure}[!htbp]
    \centering
    
      \centering
    \begin{minipage}{0.4\textwidth}
  \begin{tikzpicture}
    \begin{axis}[
        xlabel={$S$},
         ylabel={$ \text{BPD}$},
        legend pos=north east,
        major grid style={dotted,black},
        ymajorgrids=true, 
        title= (a) Some-to-Some,
        width=0.8*0.999\linewidth, 
        height=0.8*6cm,
        xtick={1,2,3,4},
        ymin=3.2,
    ]
  \addplot[only marks,
        color=C6,
        mark=diamond*,
        mark options={solid, scale=2, opacity=0.9},
        very thick
        ]
        coordinates {
        (1,3.437)
        };
\addplot[only marks,
        color=C2,
        mark options={solid, scale=2, opacity=0.9},
        mark=triangle*,
        very thick
        ]
  coordinates {
        (2,3.2410)
        };
    \addplot[only marks,
        color=C9,
        mark=*,
        mark options={solid, scale=2, opacity=0.6},
        very thick
        ]
  coordinates {
        (2, 3.2402)
        (3,3.2330)
        };
\addplot[only marks,
        color=C3,
        mark=pentagon*,
        mark options={solid, scale=2, opacity=0.9},
        very thick
        ]
  coordinates {
        (2,3.2312)
        (3,3.2289)
        (4,3.2275)
        };
    \legend{$A=1$,$A=2$, $A=3$,$A=4$}
    
    \end{axis}
  \end{tikzpicture}    
    \end{minipage}
    \begin{minipage}{0.4\textwidth}
  \begin{tikzpicture}
    \begin{axis}[
        xlabel={$S$},
         ylabel={$ \text{BPD}$},
        legend pos=north east,
        major grid style={dotted,black},
        ymajorgrids=true,
        title= (b) Some-to-All,
        width=0.8*0.999\linewidth, 
        height=0.8*6cm,
        xtick={1,2,3,4},
        ymin=3.18,
    ]
 \addplot[only marks,
        color=C6,
        mark=diamond*,
        mark options={solid, scale=2, opacity=0.9},
        very thick
        ]
        coordinates {
        (1,3.294744791)
        };
 \addplot[only marks,
        color=C2,
        mark options={solid, scale=2, opacity=0.9},
        mark=triangle*,
        very thick
        ]
  coordinates {
        (1, 3.205705628)
        (2,3.207795628)
        };
    \addplot[only marks,
        color=C9,
        mark=*,
        mark options={solid, scale=2, opacity=0.6},
        very thick
        ]
  coordinates {
        (1,3.202438871)
        (2, 3.20162391)
        (3,3.202777611)
        };
\addplot[only marks,
        color=C3,
        mark=pentagon*,
        mark options={solid, scale=2, opacity=0.9},
        very thick
        ]
  coordinates {
        (1,3.197333506)
        (2,3.1951721)
        (3,3.19618)
        (4,3.1971295)
        };
    \legend{$A=1$,$A=2$, $A=3$,$A=4$}
    \end{axis}
  \end{tikzpicture}    
    \end{minipage}
    \caption{BPD results on \textbf{CIFAR-10} for \modelname{MISVAE}s trained with different estimators and combinations of $S$ and $A$.}
    \label{fig:enter-label4}
\end{figure}

\textbf{Results.} The experiments conducted with real datasets confirm the insights gained from the Toy Experiment. In Figures \ref{fig:mnist_small_S_and_A_s2s_vs_sva} (MNIST) and \ref{fig:enter-label4} (CIFAR-10), we compare the NLL scores of \modelname{MISVAE} when trained using the $\mathrm{S2S}$, $\mathrm{S2A}$, and $\mathrm{A2A}$ estimators across progressively increasing values of $S$ and $A$. Note that the $\mathrm{A2A}$ estimator is used when $S=A$. We make two observations.
First, for a given $S$, the $\mathrm{S2S}$ estimator consistently achieves lower NLL scores as $A$ increases, with this effect being more pronounced for MNIST. Second, as depicted in Fig. \ref{fig:mnist_small_S_and_A_s2s_vs_sva}b, the epoch completion time for the S2S estimator remains constant as $A$ increases, assuming $S$ is fixed. These observations collectively suggest the potential for improving performance by significantly increasing $A$ while maintaining a small $S$. This hypothesis is further examined in Fig. \ref{fig:increasing_S_for_s2s}b, which shows that initially, increasing $A$ enhances performance. However, for large enough $A$, the performance gains start to wane, likely due to the decreasing probability of each component being updated in an epoch as $A$ increases.

We also make some useful observations regarding the S2A estimator. For a given value of $A$, the S2A estimator demonstrates approximately equivalent performance as $S$ varies, confirming its unbiasedness on both CIFAR-10 (Fig. \ref{fig:enter-label4}b) and MNIST (Fig. \ref{fig:mnist_small_S_and_A_s2s_vs_sva}a). Additionally, from Fig. \ref{fig:mnist_small_S_and_A_s2s_vs_sva}b, we observe a marginal increase in epoch completion time for a fixed \(S\) with increasing \(A\). The combination of being unbiased and efficient suggests that the S2A estimator can be scaled to a substantially larger number of mixture components. In Fig. \ref{fig:big_smnist}, we gradually increase $A$ up to hundreds of mixture components with \(S=1\) held fixed and achieve SOTA NLL scores on MNIST and FashionMNIST. 
In Table \ref{tab:NLL SOTA}, we compare the number of network parameters and the NLL scores of our models against other competitive models in this domain. Notably, our models use far fewer network parameters at superior performance in terms of NLL.

Finally, we compare MISVAE and our estimators against SEMVAE in Fig. \ref{fig:misvae_semvae_estimator_comparison}, which reveals that S2A, even with \(S=2\) held constant, either outperforms or matches the performance of SEMVAE. Also, it achieves this with only a slight increase in inference time and network parameters as \(A\) increases, unlike SEMVAE, where epoch completion time and network parameters escalate rapidly with $A$.
\begin{table}[!t]
\centering
\caption{FID scores evaluated on the MNIST test set}
\begin{tabular}{l|cccc}
\hline
S& $1$ & $2$ & $3$ & $4$ \\ \hline
FID & 10.87   & 10.02   & 9.89    & 9.70   \\ \hline
\end{tabular}
\label{tab:fid_inc_S}
\end{table}
We also see that S2S outperforms both S2A, albeit being considerably slower, and A2A, with approximately the same inference time. The superior performance of the S2S estimator is because it can be regarded as an ensemble of mixtures, thus providing a tighter bound compared to a single mixture model, as proven by \citet{kviman2022multiple}.

\begin{table*}[htbp]
    \centering
    \caption{NLL estimates on eight Phylogenetic Datasets. All VBPI methods use $1000$ importance samples, and the results are averaged over~100 runs and three independently trained models.
    %Lower standard deviations (in parentheses) imply lower variance-estimators. As such, and 
    The time evaluation for the likelihood function ($p$) and variational probability ($q$) was conducted on an i5-1130G7 CPU using one core using a CPU timer. \#$p$ and \#$q$ describe the number of times the likelihood and variational distribution needs to be evaluated respectively.}
    % {\footnotesize
     % {\tiny
    % { \small
    {\scriptsize
    \tabcolsep1pt
    \begin{tabularx}{\linewidth}{@{}>{\hsize=0.35\hsize}X>{\hsize=0.35\hsize}X>{\hsize=0.35\hsize}X>{\hsize=0.35\hsize}X|XXXXXXXX@{}}
\hline
\multicolumn{4}{X|}{Data} & DS1 & DS2 & DS3 & DS4 & DS5 & DS6 & DS7 & DS8 \\\hline
\multicolumn{4}{X|}{\# Taxa} & 27 & 29 & 36 & 41 & 50 & 50 & 59 & 64 \\
\multicolumn{4}{X|}{\# Sites} & 1949 & 2520 & 1812 & 1137 & 378 & 1133 & 1824 & 1008 \\
\multicolumn{4}{X|}{Run time $p$ (ms)} & 0.765 & 1.007 & 1.249 & 1.235 & 1.481 & 1.622 & 1.821 & 1.954 \\
\multicolumn{4}{X|}{Run time $q$ (ms)} & 0.778 & 1.082 & 1.425 & 1.427 & 1.816 & 1.652 & 2.130 & 2.359 \\
\hline
% \multicolumn{9}{c}{\textbf{VBPI with Mixtures}}\\
% \hline
% VBPI & -7108.50 (0.23)& -26367.70 (0.09)& -33735.10 (0.14)& -13330.03 (0.23)& -8214.80 (0.50)& -6724.59 (0.53)& -37332.12 (0.45)& -8652.39 (.71) \\
% Mix\textsubscript{$S=2$} & -7108.44 (.12)& -26367.71 (.06)& -33735.10 (.07)& -13330.00 (.17)& -8214.75 (.36)& -6724.54 (.31)& -37332.04 (.24)& -8651.68 (.49) \\
% Mix\textsubscript{$S=3$} & -7108.42 (.11) &	-26367.71 (.04) &	-33735.10 (.06) &	-13329.97 (.17) &	-8214.73 (.26) &	-6724.51 (.28) &	-37332.03 (.18) &	-8650.83 (.46) \\
% \hline
\multicolumn{2}{X}{$A$} &  \textbf{\#$p$} & \textbf{ \#$q$} &\multicolumn{8}{c}{\textbf{VBPI with NFs and Mixtures using (A2A)} \cite{kviman2023improved} }\\
\hline
 \multicolumn{2}{X}{1} & 1 & 1 & 7108.42 (.15)& 26367.72 (.06)& 33735.10 (.07)& 13330.00 (.23)& 8214.70 (.47)& 6724.50 (.45)& 37332.01 (.27)& 8650.68 (.46) \\
 \multicolumn{2}{X}{2} & 2 & 4 & 7108.40 (.10)& 26367.71 (.04)& 33735.10 (.05)& 13329.95 (.15)& 8214.62 (.26)& 6724.44 (.32)& 37331.96 (.19)& 8650.56 (.33) \\
 \multicolumn{2}{X}{3} & 3 & 9 & 7108.40 (.06)& 26367.70 (.03)& 33735.09 (.04)& 13329.94 (.11)& 8214.56 (.22)& 6724.40 (.23)& 37331.96 (.15)& 8650.54 (.30) \\
 \multicolumn{2}{X}{4} & 4 & 16 & 7108.40 (.05) & 26367.71 (.02) & 33735.09 (.02) & 13329.93 (.07) & 8214.57 (.16) & 6724.31 (.19) & 37331.92 (.13) & 8650.66 (.24) \\
\hline
$A$& $S$& \textbf{\#$p$}& \textbf{\#$q$} & \multicolumn{8}{c}{\textbf{VBPI with NFs and Mixtures using (S2A)}}\\
\hline
% Mix_{S=1& N=1} & 7108.43 (.17) & 26367.71 (.06) & 33735.10 (.08) & 13330.00 (.18) & 8214.69 (.39) & 6724.53 (.42) & 37332.00 (.30) & 8650.70 (.43) \\ 
 2 & 1 & 1 & 2 & 7108.54 (.11) & 26367.71 (.04) & 33735.10 (.05) & 13329.97 (.16) & 8214.92 (.44) & 6724.49 (.35) & 37331.99 (.19) & 8651.32 (.42) \\ 
%  2&2&2&4 & 7108.47 (.09) & 26367.70 (.04) & 33735.10 (.05) & 13329.96 (.13) & 8214.62 (.33) & 6724.46 (.27) & 37331.95 (.20) & 8650.67 (.35) \\ 
 3&1&1&3 & 7108.77 (.08) & 26367.71 (.03) & 33735.10 (.04) & 13330.05 (.12) & 8215.34 (.42) & 6724.49 (.29) & 37331.98 (.15) & 8651.85 (.35) \\ 
 3&2&2&6 & 7108.44 (.08) & 26367.71 (.03) & 33735.09 (.04) & 13329.94 (.10) & 8214.62 (.28) & 6724.42 (.27) & 37331.95 (.17) & 8650.76 (.29) \\ 
%  3&3&3&9 & 7108.43 (.07) & 26367.71 (.02) & 33735.09 (.03) & 13329.94 (.09) & 8214.58 (.19) & 6724.40 (.24) & 37331.94 (.16) & 8650.58 (.27) \\ 
 4&1&1&4 & 7108.92 (.09) & 26367.71 (.02) & 33735.09 (.03) & 13330.05 (.11) & 8217.98 (.73) & 6724.61 (.28) & 37331.96 (.13) & 8652.01 (.32) \\ 
 4&2&2&8 & 7108.40 (.06) & 26367.70 (.03) & 33735.10 (.03) & 13329.97 (.09) & 8214.69 (.21) & 6724.41 (.22) & 37331.98 (.14) & 8650.85 (.25) \\ 
 4&3&3&12 & 7108.40 (.05) & 26367.71 (.02) & 33735.09 (.03) & 13329.93 (.08) & 8214.59 (.17) & 6724.38 (.21) & 37331.96 (.11) & 8650.72 (.22) \\ 
\hline
$A$& $S$& \textbf{\#$p$}& \textbf{\#$q$} & \multicolumn{8}{c}{\textbf{VBPI with NFs and Mixtures using (S2S)}}\\
\hline
% Mix_{S=1& N=1} & 7108.43 (.13) & 26367.71 (.08) & 33735.09 (.11) & 13330.00 (.30) & 8214.69 (.38) & 6724.53 (.47) & 37332.02 (.24) & 8650.68 (.53) \\ 
  2&1&1&1 & 7108.41 (.10) & 26367.71 (.04) & 33735.09 (.06) & 13330.00 (.16) & 8214.77 (.38) & 6724.53 (.32) & 37332.00 (.20) & 8650.65 (.37) \\ 
% Mix_{S=2& N=2} & 7108.47 (.09) & 26367.71 (.04) & 33735.10 (.05) & 13329.94 (.14) & 8214.63 (.28) & 6724.45 (.30) & 37331.97 (.20) & 8650.65 (.39) \\ 
  3&1&1&1 & 7108.41 (.08) & 26367.71 (.04) & 33735.09 (.05) & 13330.02 (.16) & 8214.90 (.35) & 6724.53 (.26) & 37332.00 (.18) & 8650.97 (.30) \\ 
  3&2&2&4 & 7108.42 (.08) & 26367.71 (.03) & 33735.09 (.04) & 13329.96 (.11) & 8214.63 (.23) & 6724.44 (.24) & 37331.98 (.13) & 8650.71 (.27) \\ 
% Mix_{S=3& N=3} & 7108.42 (.06) & 26367.71 (.02) & 33735.09 (.03) & 13329.94 (.09) & 8214.57 (.20) & 6724.37 (.25) & 37331.95 (.13) & 8650.59 (.24) \\ 
  4&1&1&1 & 7108.74 (.15) & 26367.71 (.03) & 33735.09 (.04) & 13330.01 (.11) & 8215.73 (.48) & 6724.54 (.23) & 37332.00 (.15) & 8651.34 (.27) \\ 
  4&2&2&4 & 7108.41 (.06) & 26367.71 (.03) & 33735.09 (.04) & 13329.96 (.11) & 8214.72 (.20) & 6724.44 (.24) & 37331.98 (.13) & 8650.68 (.23) \\ 
  4&3&3&9 & 7108.40 (.06) & 26367.71 (.02) & 33735.09 (.03) & 13329.94 (.07) & 8214.59 (.16) & 6724.39 (.21) & 37331.94 (.15) & 8650.64 (.22) \\ 
% Mix_{S=4& N=4} & 7108.40 (.05) & 26367.71 (.02) & 33735.09 (.03) & 13329.93 (.07) & 8214.55 (.18) & 6724.33 (.20) & 37331.94 (.11) & 8650.57 (.19) \\
 \hline
 %\multicolumn{4}{X|}{Data} & 
 %& \\
 \multicolumn{4}{X|}{}& \multicolumn{8}{c}{\textbf{MCMC and VBPI with GNNs (scores from \cite{zhang2019variational} and \cite{zhang2023learnable})}} \\
  % \multicolumn{4}{X|}{}& \multicolumn{8}{c}{\textbf{MCMC and VBPI with GNNs (scores from \cite{zhang2019variational} and \cite{zhang2023learnable})}} \\
 \hline
  \multicolumn{4}{X|}{MrBayes\textsubscript{ss}}  & 7108.42 (.18) & 26367.57 (.48) & 33735.44 (.50) & 13330.06 (.54) & 8214.51 (.28) & 6724.07 (.86) & 37332.76 (2.42) & 8649.88 (1.75) \\
\multicolumn{4}{X|}{GGNN} & 7108.40 (.19) & 26367.73 (.10) & 33735.11 (.09) & 13329.95 (.19) & 8214.67 (.36) & 6724.38 (.42) & 37332.03 (.30) & 8650.68 (.48) \\
 \multicolumn{4}{X|}{EDGE}   & 7108.41 (.14) & 26367.73 (.07) & 33735.12 (.09) & 13329.94 (.19) & 8214.64 (.38) & 6724.37 (.40) & 37332.04 (.26) & 8650.65 (.45)
\end{tabularx}
    }
    \label{tab:ml_elbo}
    \vspace*{-1.01\baselineskip}
\end{table*}

\textbf{Generative performance.}

In order to understand how the decoder contributes to the impressive NLL scores, we trained four MISVAEs using S2A with $A = 4$ and $S = 1, …, 4$. For each model, we evaluated the decoders generative performance via the FID score on the MNIST test dataset. The results in Table~\ref{tab:fid_inc_S} demonstrate that the generative performance of the decoder increases when learned with increasingly expensive estimators. These results likely stem from the frequency of likelihood function evaluations. I.e., when \( S \leq A \), the likelihood function is evaluated less frequently during gradient calculations with respect to the decoder weights.

In Appendix \ref{app:generated_imgs} we also include visualizations of generated images from MNIST and FashionMNIST.

\subsection{Bayesian Phylogenetics}

\citet{kviman2023improved} achieved SOTA results by developing a mixture of variational phylogenetic posterior approximations, learnt via the \modelname{A2A} estimator. As empirically demonstrated in \cite{kviman2023cooperation, kviman2023improved}, the NLL improves monotonically with an increasing number of components. Using a large number of mixtures in mixture variational Bayesian phylogenetic inference (VBPI), however, lead to significant computational overhead and slow convergence, especially with large datasets. We address these issues with the \modelname{S2A} and \modelname{S2S} estimators, showing that they achieve comparable results with reduced computational requirements. Overall, our estimators enable us to minimize computational time without compromising performance.

We precisely followed the training procedure outlined by \citet{kviman2023improved} and, like them, adopted the VBPI algorithm with NFs and mixture distributions. We performed experiments on eight popular datasets for Bayesian phylogenetics \citep{Hedges1990-eu, Garey1996-ti, Yang2003-eg, Henk2003-dn, Lakner2008-ks, Zhang2001-hs, Yoder2004-ut, Rossman2001-ph}. As in \citet{zhang2019variational, zhang2020improved, moretti2021variational, koptagel2022vaiphy, zhang2022variational, zhang2023learnable}, we learn the approximations of branch-length and tree-topology distributions. As can be seen in Table \ref{tab:ml_elbo}, similar, or identical, NLL results are obtained with fewer likelihood and variational probability evaluations. The cost of each such operation is specified in the table. This implies that our estimators have successfully reduced the inference time of the SOTA model, while preserving the impressive NLL scores.

\section{Future Work}

A future avenue of research is to theoretically justify the results in this work by producing convergence results for mixtures. Convergence properties and guarantees for BBVI have previously been studied by \citet{Kim2023OnTC, domke2024provable, hotti2024benefits} for when the variational family belongs to the location-scale family, which does not include mixtures. Since the mixture differential entropy can no longer be computed in closed form, one would need to consider stochastic estimates of both the energy and the differential entropy gradients of the variational objective. In this context, it would be interesting to consider the sticking-the-landing (STL) estimator, previously studied in the setting of BBVI by both  \cite{kim2024linear} and \cite{domke2024provable}. It turns out that with the STL estimator, a linear convergence rate can be achieved when the variational family contains the true posterior, which approximately holds for a mixture of Gaussians given a sufficient number of components.

Gradient-based inference of mixture weights in BBVI is non-trivial \cite{morningstar2021automatic}, and there are multiple approaches to reparameterized sampling of the components \cite{figurnov2018implicit, morningstar2021automatic}. However, an alternative is resampling-based inference, as in the adaptive IS literature. Recently, \citet{kviman2024variational} proposed a new resampling methodology which could be applied in mixture BBVI to weight components such that the ELBO is maximized via a combinatorial optimization algorithm.

\section{Conclusion}

In this work, we have addressed the scalability and efficiency challenges faced by mixtures in BBVI, by introducing MISVAE and the novel estimators of the MISELBO: the Some-to-All and Some-to-Some estimators. Our contributions significantly decrease the number required learnable parameters and the computational costs associated with increasing the number of mixture components, enabling scalability of mixture models without compromising performance.

\section*{Impact Statement}

This paper presents work with the goal to advance the field of machine learning. There are many potential societal consequences of our work, none which we feel must be specifically highlighted here.

\section*{Acknowledgments}
First, we acknowledge the insightful comments provided by the reviewers, which have helped
improve our work. This project was made possible through funding from the Swedish Foundation for
Strategic Research grants BD15-0043 and ID19-0052, and from the Swedish Research Council grant 2018-05417\_VR.
The computations and data handling were enabled by resources provided by the Swedish National
Infrastructure for Computing (SNIC), partially funded by the Swedish Research Council through
grant agreement no. 2018-05973.

\bibliography{main}
\bibliographystyle{icml2024}

\newpage
\appendix
\onecolumn

\section{Expected Values of the Estimators} \label{sec:upper_bounds}
\label{app:exp_values}
Here we prove the results presented in Section~\ref{sec:estimators}.

There are $A$ components in total and thus ${A \choose S}$ subsets $\Phi$ of cardinality $S$ (without replacement). Furthermore, summed over all subsets, an arbitrary component, $q_{\phi_k}$, is observed ${A \choose S} \frac{S}{A} = {A-1 \choose S-1}$ times. Define a uniform distribution in the space of all $S$-subsets, $\Omega^S$,
 \begin{equation}
 \varphi\left(\Phi\right)  = \frac{1}{{A \choose 
S}}
 \end{equation}
 and the \modelname{Some-to-All} estimator as
 \begin{align}
 \label{eq:s2a_app}
\tilde{\mathcal{L}}_\text{\modelname{Some-to-All}}^{M,S,J} := \frac{1}{M}\sum_{m=1}^M\frac{1}{S}\sum_{\phi_{k}\in \Phi^{m}}\frac{1}{J}\sum_{j'=1}^J\log\frac{p(x, z^{j'}_{k})}{\frac{1}{A}\sum^A_{j=1} q_{\phi_{j}}(z^{j'}_{k}|x)}, \quad \Phi^{1},...,\Phi^{M}\sim \varphi(\Phi),\quad z_k^{1},...,z_k^{J} \sim q_{\phi_k}(z|x).
\end{align}

\newtheorem*{thm41}{Theorem \ref{stwoaunbiasedcorollary}}
\begin{thm41}
\stwoaunbiasedthm
\end{thm41}

\begin{proof}
    Taking expectations and utilizing that the expectation is a linear operator, the R.H.S. of Eq. \eqref{eq:s2a_app} is 
    \begin{align}
        &\frac{1}{M}\sum_{m=1}^M
        \mathbb{E}_{\varphi(\Phi)}\left[
        \frac{1}{S}\sum_{\phi_{k}\in \Phi}\frac{1}{J}\sum_{j'=1}^J
        \mathbb{E}_{q_{\phi_{k}}(z|x)}\left[
        \log\frac{p(x, z_{k})}{\frac{1}{A}\sum^A_{j=1} q_{\phi_{j}}(z_{k}|x)}\right]\right]= \\
        & \mathbb{E}_{\varphi(\Phi)}\left[
        \frac{1}{S}\sum_{{\phi_{k}}\in \Phi}
        \mathbb{E}_{q_{\phi_{k}}(z|x)}\left[
        \log\frac{p(x, z_{k})}{\frac{1}{A}\sum^A_{j=1} q_{\phi_{j}}(z_{k}|x)}\right]\right]=\label{expression_1}\\
    &\sum_{\Phi\in\Omega^S} \frac{1}{{A \choose S}}
        \frac{1}{S}\sum_{{\phi_{k}}\in \Phi}
        \mathbb{E}_{q_{\phi_{k}}(z|x)}\left[
        \log\frac{p(x, z_{k})}{\frac{1}{A}\sum^A_{j=1} q_{\phi_{j}}(z_{k}|x)}\right]=\label{equality_s2a}\\
         & \frac{1}{{A \choose S}}
        \frac{1}{S}\sum_{{\phi_{k}}\in \Phi}
        \mathbb{E}_{q_{\phi_{k}}(z|x)}\left[
        \log\frac{p(x, z_{k})}{\frac{1}{A}\sum^A_{j=1} q_{\phi_{j}}(z_{k}|x)}\right]=\label{equality_s2a}\\
        & \frac{1}{{A \choose S}}
        \frac{1}{S}{A -1  \choose S -1}\sum_{k=1}^A
        \mathbb{E}_{q_{\phi_{k}}(z|x)}\left[
        \log\frac{p(x, z_{k})}{\frac{1}{A}\sum^A_{j=1} q_{\phi_{j}}(z_{k}|x)}\right]=\\
        &\frac{1}{A}\sum_{k=1}^A
        \mathbb{E}_{q_{\phi_{k}}(z|x)}\left[
        \log\frac{p(x, z_{k})}{\frac{1}{A}\sum^A_{j=1} q_{\phi_{j}}(z_{k}|x)}\right] = \mathcal{L}_\text{MIS},
    \end{align}
    where the equality in Eq. \eqref{equality_s2a} holds as $$\sum_{\Phi\in \Omega^S}\sum_{\phi_k\in\Phi} q_{\phi_k} = \sum_{\Phi\in \Omega^S}\sum_{k=1}^A \mathbbm{1}_{\{\phi_k \in \Phi\}}(\phi_k)q_{\phi_k}(z|x) = {A -1  \choose S -1}\sum_{k=1}^Sq_{\phi_k}(z|x),$$ following the statement in the beginning of this section; component $q_{\phi_k}$ will be observed ${A\choose S}\frac{S}{A}$ times in all possible subsets.
\end{proof}

From the theorem above we can provide the following corollary.

\newtheorem*{corollary2}{ \textbf{Corollary \ref{stwoalower_bound}}}
\begin{corollary2}
\newstwoaunbiasedcorollary
\end{corollary2}

\begin{proof}
    From Theorem \ref{stwoaunbiasedcorollary}, we have $\mathbb{E}\left[\tilde{\mathcal{L}}_\text{\modelname{S2A}}^{M,S,J}\right] = \mathcal{L}_\text{MIS}$ and from \citet{kviman2022multiple} it is known that $\mathcal{L}_\text{MIS} \leq \log p_\theta(x)$.
\end{proof}

Next, we turn to the examination of the expected value of the \modelname{Some-to-Some} estimator
\begin{equation}
    \tilde{\mathcal{L}}_\text{\modelname{S2S}}^{M,S,J} := \frac{1}{M}\sum_{m=1}^M\frac{1}{S}\sum_{\phi_{k}\in \Phi^{m}}\frac{1}{J}\sum_{j'=1}^J\log\frac{p(x, z^{j'}_{k})}{\frac{1}{S}\sum_{\phi_{j} \in \Phi^m} q_{\phi_{j}}(z^{j'}_{k}|x)}, \quad \Phi^{1},...,\Phi^{M}\sim \varphi(\Phi),\quad z_k^{1},...,z_k^{J} \sim q_{\phi_k}(z|x).
\end{equation}
Note that this estimator diverges from the \modelname{S2A} estimator merely in the denominator inside the logarithm. However, this change clearly implies that the \modelname{S2S} estimator is not an unbiased estimator of Eq. \eqref{eq:miselbo}.
In fact, its expected value is instead a lower bound on MISELBO, as we will show here. 
\begin{theorem}
\label{thm:s2sleqmiselbo}
    The expected value of the \modelname{Some-to-Some} estimator is a lower bound on MISELBO, i.e.
    \begin{align*}
        \mathbb{E}\left[
    \tilde{\mathcal{L}}_\text{\modelname{S2S}}^{M,S,J}
        \right]\leq \mathcal{L}_\text{MIS}.
    \end{align*}
\end{theorem}
\begin{proof}
    Using that $\mathbb{E}\left[
    \tilde{\mathcal{L}}_\text{\modelname{S2A}}^{M,S,J}
        \right] = \mathcal{L}_\text{MIS}$, we will directly check if $\mathbb{E}\left[
    \tilde{\mathcal{L}}_\text{\modelname{S2S}}^{M,S,J}
        \right] \leq \mathbb{E}\left[
    \tilde{\mathcal{L}}_\text{\modelname{S2A}}^{M,S,J}
        \right]$. From the formulation in Eq. \eqref{expression_1}, the inequality can equivalently be expressed as
        \begin{align}
           \mathbb{E}_{\varphi(\Phi)}\left[
        \frac{1}{S}\sum_{{\phi_{s}}\in \Phi}
        \mathbb{E}_{q_{\phi_{s}}(z_s|x)}\left[
        \log\frac{p(x, z_{s})}{\frac{1}{S}\sum_{\phi_k \in \Phi} q_{\phi_{k}}(z_{s}|x)}\right]\right]  \leq \mathbb{E}_{\varphi(\Phi)}\left[
        \frac{1}{S}\sum_{{\phi_{s}}\in \Phi}
        \mathbb{E}_{q_{\phi_{s}}(z_s|x)}\left[
        \log\frac{p(x, z_{s})}{\frac{1}{A}\sum^A_{j=1} q_{\phi_{j}}(z_{s}|x)}\right]\right].
        \end{align}
        Subtracting the R.H.S. with the L.H.S., we get
        \begin{align}
           &\mathbb{E}_{\varphi(\Phi)}\left[
        \frac{1}{S}\sum_{{\phi_{s}}\in \Phi}
        \mathbb{E}_{q_{\phi_{s}}(z_s|x)}\left[
        \log\frac{p(x, z_{s})}{\frac{1}{A}\sum^A_{j=1} q_{\phi_{j}}(z_{s}|x)}\right]\right] - \mathbb{E}_{\varphi(\Phi)}\left[
        \frac{1}{S}\sum_{{\phi_{s}}\in \Phi}
        \mathbb{E}_{q_{\phi_{s}}(z_s|x)}\left[
        \log\frac{p(x, z_{s})}{\frac{1}{S}\sum_{\phi_k \in \Phi} q_{\phi_{k}}(z_{s}|x)}\right]\right]\\
        =& \mathbb{E}_{\varphi(\Phi)}\left[\frac{1}{S}
        \sum_{{\phi_{s}}\in \Phi}
        \mathbb{E}_{q_{\phi_{s}}(z_s|x)}\left[
        \log\frac{p(x, z_{s})}{\frac{1}{A}\sum^A_{j=1} q_{\phi_{j}}(z_{s}|x)} - 
        \log\frac{p(x, z_{s})}{\frac{1}{S}\sum_{\phi_k \in \Phi} q_{\phi_{k}}(z_{s}|x)}\right]\right]\\
        =&\mathbb{E}_{\varphi(\Phi)}\left[\frac{1}{S}
        \sum_{{\phi_{s}}\in \Phi}
        \mathbb{E}_{q_{\phi_{s}}(z_s|x)}\left[
        \log\frac{\frac{1}{S}\sum_{\phi_k \in \Phi} q_{\phi_{k}}(z_{s}|x)}{\frac{1}{A}\sum^A_{j=1} q_{\phi_{j}}(z_{s}|x)}\right]\right]\\ =&
        \mathbb{E}_{\varphi(\Phi)}\left[
        \text{KL}\left(
        \frac{1}{S}\sum_{\phi_k \in \Phi}q_{\phi_{k}}(z|x)\Bigg\Vert \frac{1}{A}\sum^A_{j=1} q_{\phi_{j}}(z|x)
        \right)
        \right] \geq 0,
        \end{align}
        where the final inequality holds as the KL, and thus the average of KLs, is non-negative.
\end{proof}

From the result above, it furthermore follows directly that the expected value of the \modelname{S2S} estimator is a lower bound on the marginal log-likelihood.
\begin{corollary}
    The expected value of the \modelname{Some-to-Some} estimator is a lower bound on the marginal log-likelihood, i.e.
    \begin{align*}
        \mathbb{E}\left[
    \tilde{\mathcal{L}}_\text{\modelname{S2S}}^{M,S,J}
        \right]\leq \log p_\theta(x).
    \end{align*}
\end{corollary}

\newtheorem*{corollary1}{ \textbf{Corollary \ref{stwoaunbiasedcorollary}}}
\begin{corollary1}
\newstwoaunbiasedcorollary
\end{corollary1}

\begin{proof}
    Recalling that $\mathcal{L}_\text{MIS}\leq \log p_\theta(x)$, it follows from Theorem \ref{thm:s2sleqmiselbo} that 
    $\mathbb{E}\left[
    \tilde{\mathcal{L}}_\text{\modelname{S2S}}^{M,S,J}
        \right]\leq \mathcal{L}_\text{MIS}\leq\log p_\theta(x)$.
\end{proof}

\section{Extension to Weighted Mixture}
\label{app:weighted_proof}
We now extend the unbiasedness result to the case of weighted mixtures. We repeat Theorem \ref{w_thm} from the main text. 
\begin{theorem}
    The Some-to-All estimator is an unbiased estimator of MISELBO for arbitrary mixture weights.
\end{theorem}
\begin{proof}
First let us define the {\em weighted} A2A estimator
\begin{equation}
\widetilde{\mathcal{L}}_\textnormal{{A2A}} = \sum_{k=1}^A \frac{w_k}{w_A} \log \frac{p_{\theta}(x|z_{k})p_\theta(z_{k})}{\sum_{{a^\prime}=1}^A  \frac{w_{a^\prime}}{w_A} q_{\phi_{a^\prime}}(z_k|x)},
\end{equation}
where $z_{a} \sim q_{\phi_{a}(z|x)}$, $w_a$ are the unnormalized weights and 
\[
w_A := \sum_{k=1}^A w_k.
\]

Let $(I, \Sigma)$ be a measurable space with sample space 
\begin{equation}
    I = \left\{ H \in \{0, 1\}^A :  \sum_{j\in[A]} H_j = S \right\}
\end{equation}
 and probability measure $\varphi(H)$.
 
 For the sake of generality we will consider an estimator of the importance weighted MISELBO (i.e. with $L\geq1$). Let
\[
\mathcal{L}_{S2A}^{M,S,\varphi} := \frac{1}{M} \sum^M_{m=1} \sum_{k = 1}^A u_k H_k^m  \log \frac{1}{L} \sum^L_{l=1}\frac{p(x,z_k^{l})}{ \sum_{a'=1}^A \frac{w_{a'}}{w_A} q_{\phi_{a'}}(z_k^{l} \vert x)} \quad(*)
\] 
where 

\[
H^1,\ldots, H^M \sim \varphi, \quad \quad z_k^{l}\sim q_{\phi_k}(z \vert x), \quad \text{and} \quad u_k := \frac{w_k}{w_A \mathbb{E}[H_k]}
\]

Now we show that the expectation of $(*)$ is equal to the importance weighted MISELBO objective with arbitrary mixture weights.
\[
\mathbb{E}[\mathcal{L}_{S2A}^{M,S,\varphi}] = \mathbb{E}_{H^1,\ldots,H^M \sim \varphi} \left[ 
\frac{1}{M} \sum^M_{m=1} \sum_{k=1}^A u_k H_k^m 
\mathbb{E}_{ z^{l}_k \sim q_{\phi_k}(z\vert x) } \left[ \log \frac{1}{L} \sum^L_{l=1}\frac{p(x,z^{l}_k)}{ \sum_{a'} \frac{w_{a'}}{w_A} q_{\phi_{a'}}(z^{l}_k \vert x)} \right] \right]
\]

\[
=\mathbb{E}_{H\sim \varphi } \left[  \sum_{k = 1}^A  u_k {H_k} 
\mathbb{E}_{z_k^l}\left[  \log \frac{1}{L} \sum^L_{l=1}\frac{p(x,z_k^l)}{ \sum_{a'} \frac{w_{a'}}{w_A} q_{\phi_{a'}}(z_k^l \vert x)} \right] \right]
\]

\[=
\sum_{H\in I} \varphi(H)  \left[  \sum_{k=1}^A u_k  H_k 
\mathbb{E}_{z_k^l}\left[  \log \frac{1}{L} \sum^L_{l=1}\frac{p(x,z_k^l)}{ \sum_{a'} \frac{w_{a'}}{w_A} q_{\phi_{a'}}(z_k^l \vert x)} \right] \right]
\]

\[=
  \sum_{k = 1}^A u_k \left( \sum_{H\in I^S} \varphi(H) H_k \right)
\mathbb{E}_{z_k^l}\left[  \log \frac{1}{L} \sum^L_{l=1}\frac{p(x,z_k^l)}{ \sum_{a'} \frac{w_{a'}}{w_A} q_{\phi_{a'}}(z_k^l \vert x)} \right] 
\]

\[=
  \sum_{k=1}^A u_k \mathbb{E}_\varphi [H_k] \mathbb{E}_{z_k^l}\left[  \log \frac{1}{L} \sum^L_{l=1}\frac{p(x,z_k^l)}{ \sum_{a'} \frac{w_{a'}}{w_A} q_{\phi_{a'}}(z_k^l \vert x)} \right]
\]

\[=
  \sum_{k=1}^A \frac{w_k}{w_A} \mathbb{E}_{z_k^l}\left[  \log \frac{1}{L} \sum^L_{l=1}\frac{p(x,z_k^l)}{ \sum_{a'} \frac{w_{a'}}{w_A} q_{\phi_{a'}}(z_k^l \vert x)} \right].
\]
That is the same as that of the weighted A2A.
\end{proof}

\section{Block Diagrams MISVAE - S2A and A2A}

In Fig. \ref{fig:misvae-architecture-s2a-a2a} we display block diagrams for \modelname{MISVAE} with the $\mathrm{S2A}$ and $\mathrm{A2A}$ estimators.
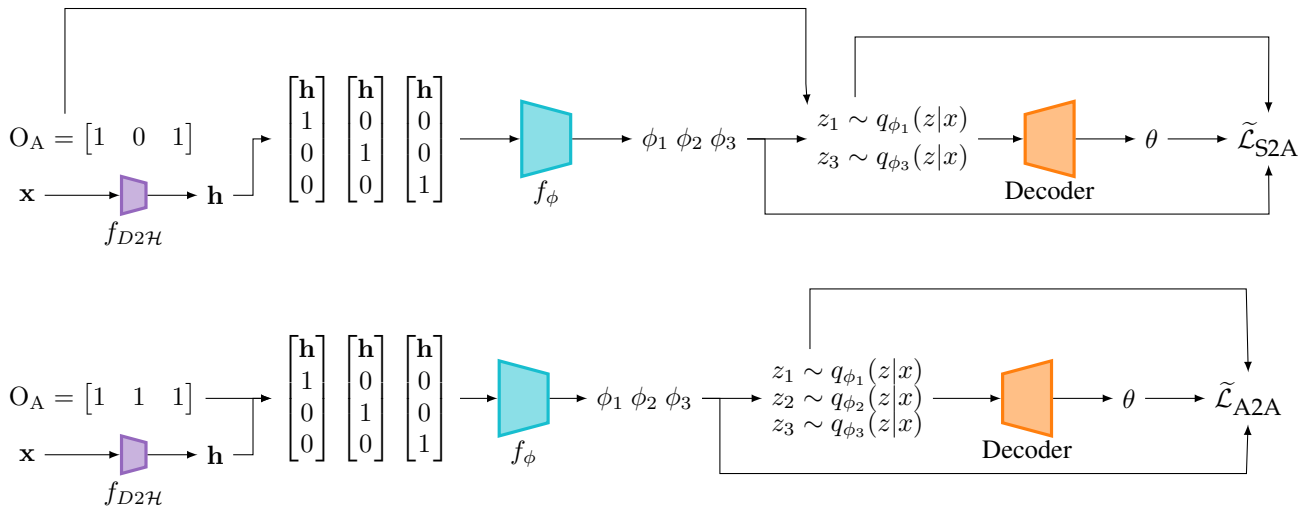
\begin{figure}[H]

    \centering
\begin{tikzpicture}
    \node[yshift=0.25cm, xshift = -6cm] (onehot) {$\mathrm{O}_\mathrm{A} = \begin{bmatrix} 1 & 0 & 1 \end{bmatrix}$};
    \node[yshift=-0.5cm, xshift = -7cm] (data) {$\bx$};
    \node[yshift=-0.5cm, xshift = -4.5cm] (hidden) {$\bh$};
    \node[yshift=0.25cm, xshift = -2.5cm] (twoonehots) {$\begin{bmatrix} \bh\\1\\ 0 \\ 0 \end{bmatrix}$ $\begin{bmatrix} \bh\\0\\ 1 \\ 0 \end{bmatrix}$ $\begin{bmatrix} \bh\\0\\ 0 \\ 1 \end{bmatrix}$ };
    
    \coordinate (input); 
    \node[draw,very thick,C4, fill=C4!50, trapezium, trapezium angle=75, shape border rotate=270, minimum
    width=0.5cm, minimum height=0.25cm, right=1cm of data, label=below:$f_{D2\mathcal{H}}$] (fd2h) {};
    \node[draw,very thick,C9, fill=C9!50, trapezium, trapezium angle=75, shape border rotate=270, minimum
        width=1cm, minimum height=0.5cm, right=0.8cm of twoonehots, label=below:$f_\phi$] (fphi) {};
        
    \node[right=0.8cm of fphi, yshift = 0cm] (phis) {$\phi_1$ $\phi_2$ $\phi_3$};

    \node[right=0.8cm of phis, yshift = -0.25cm] (qphi1) {$z_3 \sim q_{\phi_3}(z \vert x)$};
    \node[right=0.8cm of phis, yshift =0.25cm] (qphi3) {$z_1 \sim q_{\phi_1}(z \vert x)$};

     \node[draw,very thick,C1, fill=C1!50, trapezium, trapezium angle=75, shape border rotate=90, minimum
        width=1cm, minimum height=0.5cm, right=6cm of fphi, label=below:Decoder] (decoder) {};

    \node[right=0.8cm of decoder, yshift =0cm] (theta) {$\theta$};
        
    \node[right=0.8cm of theta, yshift =0cm] (elbo) {$\widetilde{\mathcal{L}}_\text{\modelname{S2A}}$};
    
    \draw[-latex] (data) to ([yshift=0.0cm]fd2h.west);
    \draw[-latex] (fd2h) to ([yshift=0.0cm]hidden.west);

    \draw[-latex] (hidden) -- ++(0.5,0) |- (twoonehots);
      \draw[-latex] (phis) -- ++(1,0) -- ++(0,-1)-- ++(6.7,0) -- (elbo);
      
        \draw[-latex] ([xshift=-0.5cm]qphi3.north) -- ++(0,0.4) -- ++(0,0.4) -- ++(5.5,0) -- (elbo);
 
 \draw[-latex] ([xshift=-0.5cm]onehot.north) -- ++(0,0.4) -- ++(0,1) -- ++(9.85,0) -- ++(0,-1)-- ([xshift=0cm, yshift=0.7cm]qphi1.west) ;
    
    \draw[-latex] (theta) to (elbo);
   %\draw[-latex] (onehot) to (twoonehots);
    \draw[-latex] (twoonehots) to (fphi);
    \draw[-latex] (fphi) to (phis);
    \draw[-latex] (phis) to ([yshift=0.25cm]qphi1.west);
    \draw[-latex] ([yshift=0.25cm]qphi1.east) to (decoder);
    %\draw[-latex] (qphi3) to (decoder);
    \draw[-latex] (decoder) to (theta);

\end{tikzpicture}
\begin{tikzpicture}
    \node[yshift=0.25cm, xshift = -6cm] (onehot) {$\mathrm{O}_\mathrm{A} = \begin{bmatrix} 1 & 1 & 1 \end{bmatrix}$};
    \node[yshift=-0.5cm, xshift = -7cm] (data) {$\bx$};
    \node[yshift=-0.5cm, xshift = -4.5cm] (hidden) {$\bh$};
    \node[yshift=0.25cm, xshift = -2.5cm] (twoonehots) {$\begin{bmatrix} \bh\\1\\ 0 \\ 0 \end{bmatrix}$ $\begin{bmatrix} \bh\\0\\ 1 \\ 0 \end{bmatrix}$  $\begin{bmatrix} \bh\\0\\ 0 \\ 1 \end{bmatrix}$};
    
    \coordinate (input); 
    \node[draw,very thick,C4, fill=C4!50, trapezium, trapezium angle=75, shape border rotate=270, minimum
    width=0.5cm, minimum height=0.25cm, right=1cm of data, label=below:$f_{D2\mathcal{H}}$] (fd2h) {};
    \node[draw,very thick,C9, fill=C9!50, trapezium, trapezium angle=75, shape border rotate=270, minimum
        width=1cm, minimum height=0.5cm, right=0.5cm of twoonehots, label=below:$f_\phi$] (fphi) {};
        
    \node[right=0.5cm of fphi, yshift = 0cm] (phis) {$\phi_1$ $\phi_2$ $\phi_3$};

    \node[right=0.8cm of phis, yshift = -0.35cm] (qphi1) {$z_3 \sim q_{\phi_3}(z \vert x)$};
     \node[right=0.8cm of phis, yshift = -0] (qphi2) {$z_2 \sim q_{\phi_2}(z \vert x)$};
    \node[right=0.8cm of phis, yshift =0.35cm] (qphi3) {$z_1 \sim q_{\phi_1}(z \vert x)$};

     \node[draw,very thick,C1, fill=C1!50, trapezium, trapezium angle=75, shape border rotate=90, minimum
        width=1cm, minimum height=0.5cm, right=6cm of fphi, label=below:Decoder] (decoder) {};
    \node[right=0.8cm of decoder, yshift =0cm] (theta) {$\theta$};
        
    \node[right=0.8cm of theta, yshift =0cm] (elbo) {$\widetilde{\mathcal{L}}_\text{\modelname{A2A}}$};
    
    \draw[-latex] (data) to ([yshift=0.0cm]fd2h.west);
    \draw[-latex] (fd2h) to ([yshift=0.0cm]hidden.west);

    \draw[-latex] (hidden) -- ++(0.5,0) |- (twoonehots);
      \draw[-latex] (phis) -- ++(1,0) -- ++(0,-1)-- ++(7,0) -- (elbo);
      
        \draw[-latex] ([xshift=-0.5cm]qphi3.north) -- ++(0,0.4) -- ++(0,0.4) -- ++(5.8,0) -- (elbo);

    \draw[-latex] (theta) to (elbo);
    \draw[-latex] (onehot) to (twoonehots);
    \draw[-latex] (twoonehots) to (fphi);
    \draw[-latex] (fphi) to (phis);
    \draw[-latex] (phis) to (qphi2);
    \draw[-latex] (qphi2.east) to (decoder);
    %\draw[-latex] (qphi3) to (decoder);
    \draw[-latex] (decoder) to (theta);

\end{tikzpicture}
\caption{Block diagram illustrating the estimation of MISELBO with \modelname{MISVAE}, showcasing the $\mathrm{S2A}$ estimator (top) with $S=2$ and $A=3$, alongside the $\mathrm{A2A}$ estimator (bottom) with $A=3$.}
\label{fig:misvae-architecture-s2a-a2a}
\vspace{-0.1cm}
\end{figure}

\section{Additional Experimental Details}
\subsection{Training Infrastructure}
All experiments were conducted on a NVIDIA RTX $4090$s with $24$ GiB of memory each using the PyTorch framework \citep{pytorch}.

\textbf{MNIST \cite{lecun-mnisthandwrittendigit-2010}.} When training on the MNIST image dataset, $f_{D2\mathcal{H}}$ was defined as a sequence of five gated convolutional layers. To learn $f_\phi$, which amortizes the mean and covariance matrices of the variational posterior, we employed two separate non-linear networks. Each network consisted of an input layer, an output layer, and a hidden layer, the latter featuring 40 latent dimensions equipped with ReLU activation functions. We used a single layer Pixel CNN decoder. For optimization, we used Adam \cite{kingma2017adam}, with a learning rate of $0.0005$, and a batch size of $100$ and initiated the process with a KL-warmup phase lasting $100$ epochs. 

\textbf{CIFAR-10 \cite{cifar}.} For $f_{D2\mathcal{H}}$ we used a pre-trained ResNet model\footnote{https://github.com/akamaster/pytorch\_resnet\_cifar10/tree/master} with $20$ layers. $f_\phi$ was defined the same way as on MNIST, except we used $128$ latent dimensions. We optimized using Adam, with a learning rate of $0.001$, and a batch size of $100$ and initiated the training with KL-warmup during $500$ epochs. We used a Pixel CNN decoder with $4$ layers.

\section{Additional Experimental Results}\label{app:exp_det}

\subsection{Toy Experiment}

\begin{figure}[H]
   \centering 
     \includegraphics[width=0.4\linewidth]{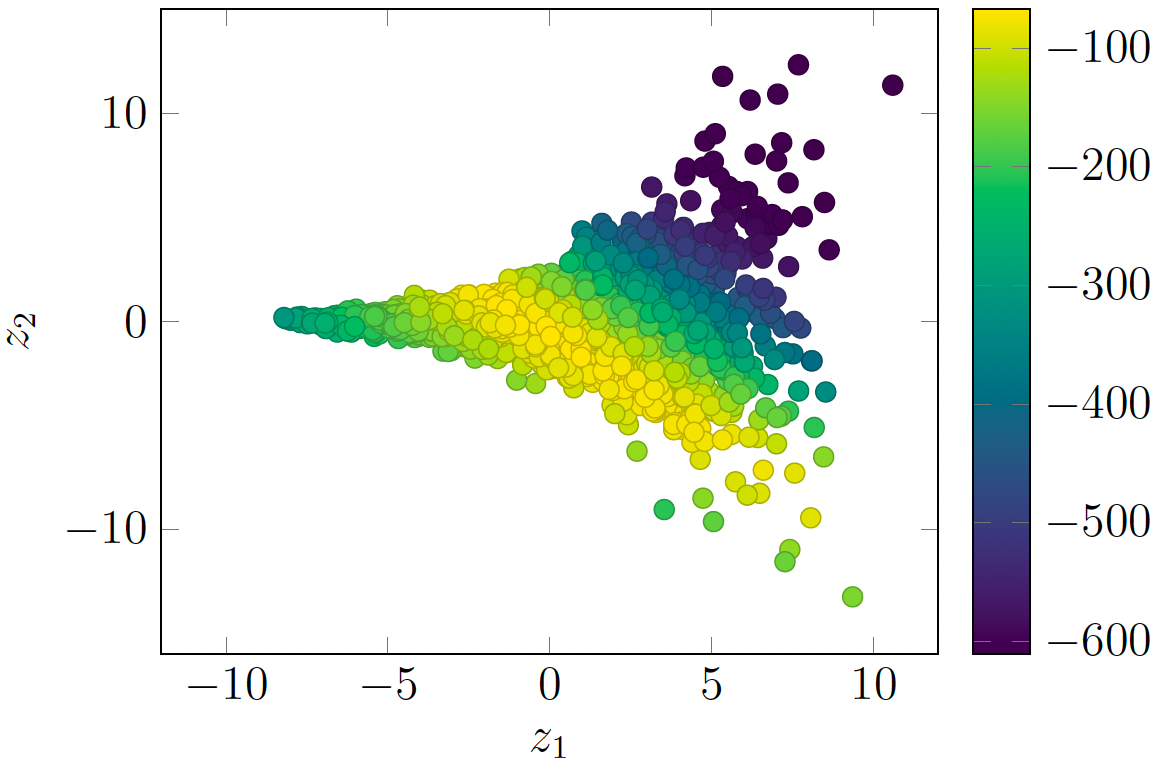}
    \caption{An unnormalized posterior from the toy example when $d_x=20$ and $N=5$. The plot is generated by sampling from the prior and evaluating the unnormalized log-probabilities of each sample (darker is lower).}
    \label{fig:ar_toy_posterior}
    \vspace{-5pt}
\end{figure}

\subsubsection{The Exponential-Decay-Bernoulli-Likelihood-and-Neal-Funnel-Prior Model}
\label{app:toy_example}
We generated a dataset $\mathcal{D}=\{x_n\}_{n=1}^5$ by sampling from the model. The approximations were learned using the Adam optimizer \cite{kingma2014adam} with learning rate equal to 0.001. The other optimizer parameters were set to their (PyTorch) default values. The variational parameters were initialized using the Kaiming-uniform initialization \cite{he2015delving}, the number of training iterations were 50k. All estimators used the same number of importance samples when estimating the MISELBO scores shown in the curve figures.

\subsubsection{Toy Experiment - Latent Space Visualization}
The approximations are visualized in the latent spaces shown in figures \ref{fig:latent_vis1}-\ref{fig:latent_vis4}. Notably, when $A=5$, the approximations from \modelname{S2A} and \modelname{A2A} are identical (recall that \modelname{S2A} required substantially less inference time, see in Sec. \ref{sec:experiments}). Meanwhile, the components in the \modelname{S2S} approximations were not able to sufficiently separate themselves.

Finally, in Fig. \ref{fig:latent_vis4}, we visualize how the approximation learned with the \modelname{S2S} estimator behaves when $S=5$ and $S=20$. The approximation obtains impressive MISELBO scores (shown in Sec. \ref{sec:experiments}) although the components largely overlapping.

\begin{figure}[H]
    \centering
    \includegraphics[width=0.5\linewidth]{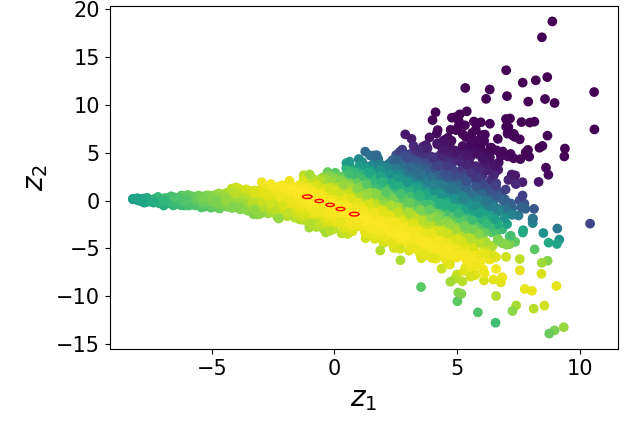}
    \caption{$dx=20$: approximation when using the \modelname{S2A} estimator, $S=2$ and $A=5$.}
    \label{fig:latent_vis1}
\end{figure}
\begin{figure}
    \centering
    \includegraphics[width=0.5\linewidth]{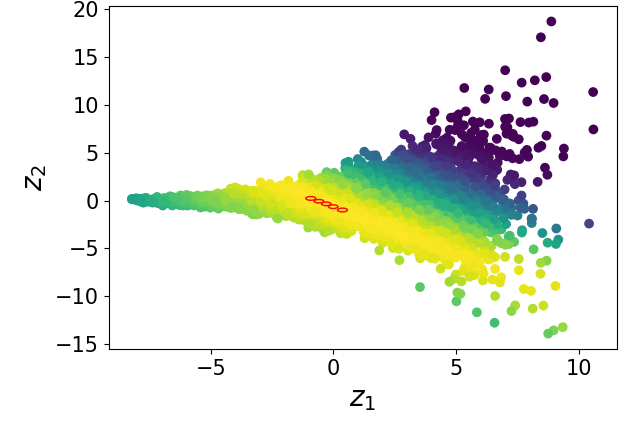}
    \caption{$dx=20$: approximation when using the \modelname{S2S} estimator, $S=2$ and $A=5$.}
    \label{fig:enter-label1}
\end{figure}
\begin{figure}
    \centering
    \includegraphics[width=0.5\linewidth]{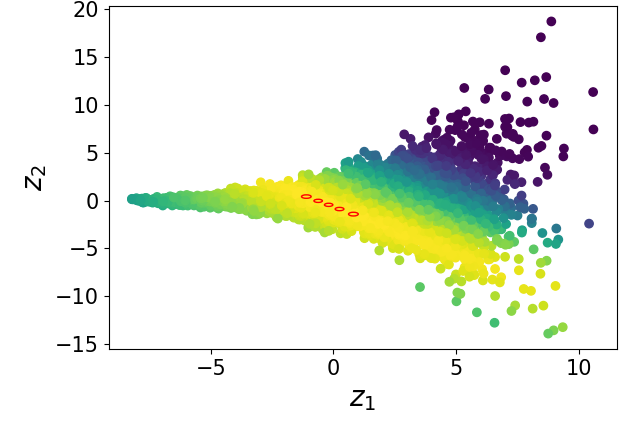}
    \caption{$dx=20$: approximation when using the \modelname{A2A} estimator, $S=5$ and $A=5$.}
    \label{fig:enter-label2}
\end{figure}

\begin{figure}
    \centering
    \includegraphics[width=0.5\linewidth]{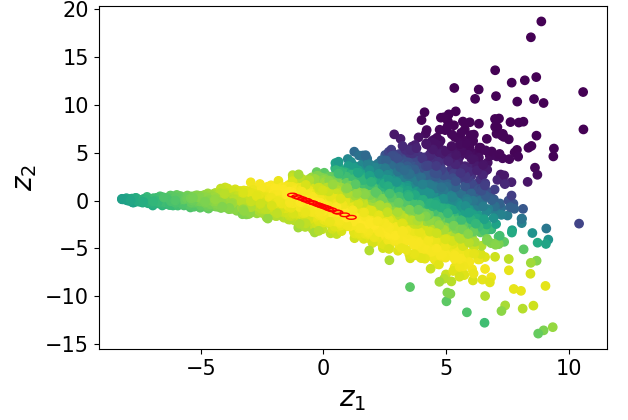}
    \caption{$dx=20$: approximation when using the \modelname{S2S} estimator, $S=5$ and $A=20$.}
    \label{fig:latent_vis4}
\end{figure}

\subsection{Additional $\mathrm{S2S}$ Results with big $A$}

In Fig. \ref{fig:increasing_S_for_s2s}a, we showcase how the final test set Negative Log Likelihood (NLL) value is impacted by setting the total number of mixtures to a large fixed value \( A=50 \) and gradually increasing the number of components used by the \( \text{S2S} \) estimator. The performance of \( \text{S2S} \) is compared to the \( \text{A2A} \) estimator, where we instead let \( A=S \) of the former estimator and gradually increase \( A \). For small values of \( S \), \( \text{S2S} \) exhibits a clear performance advantage both in terms of NLL and inference time per epoch (not shown here). However, as \( S \) approaches \( A \), the NLL performance advantage of \( \text{S2S} \) diminishes compared to \( \text{A2A} \).

In Fig. \ref{fig:increasing_S_for_s2s}b, we let \( S \) be a fixed small value and gradually increase \( A \) for the \( \mathrm{S2S} \) estimator. For both curves, an initial increase in \( A \) results in significant performance gains in terms of NLL; however, beyond a certain point, adding more mixtures yields no further improvements.

\begin{figure}[ht]
\centering

\begin{minipage}{0.4\textwidth}
\centering
(a)  \\
 \vspace{5pt}
  \centering
  \begin{tikzpicture}[trim axis left, trim axis right]
    \begin{axis}[
        xlabel={$S$},
         ylabel={$  -\log p_{\theta}(x)$},
        legend pos=north east,
        major grid style={dotted,black},
        ymajorgrids=true, 
        width=0.999\linewidth, 
        legend style={at={(0.5,-0.3)},
        anchor=north,legend columns=2, font=\small},
    ]
 \addplot[
        color=\Sfour,
        mark options={solid},
        mark=*,
        very thick
        ]
  coordinates {
        (2,79.05084711)
         (4,78.49139702)
         (10,77.87825156)
         (20,76.6404249)
        };
         \addlegendentry{\modelname{S2S} with $A=50$}
                                  \addplot[
        color=\Stwo,
        mark options={solid},
        mark=*,
        very thick
        ]
  coordinates {
        (2,78.59316406)
         (4,78.19854844)
         (10,77.69113906)
         (20,76.67368125)
        };
        \addlegendentry{\modelname{A2A} with $A=S$}
    \end{axis}
      
  \end{tikzpicture}
\end{minipage}
\begin{minipage}{0.4\textwidth}
\centering
(b) \\
 \vspace{5pt}
  \centering
  \begin{tikzpicture}[trim axis left, trim axis right]
    \begin{axis}[
        xlabel={$A$},
        ylabel={$-\log p_{\theta}(x)$},
        legend pos=north east,
        major grid style={dotted,black},
        ymajorgrids=true, 
        width=0.999\linewidth, 
        legend style={at={(0.5,-0.3)},
        anchor=north,legend columns=2, font=\small},
    ]

    \addplot[
        color=\Sone,
        mark options={solid},
        mark=*,
        very thick
    ]
    coordinates {
        (2,79.65321958)
        (3, 79.61649998)
        (4,79.64138781)
        (8, 78.57191406)
        (10, 78.58616554)
        (12, 78.59508802)
        (20,78.61388918)
        (50, 78.6070626)
    };
    \addlegendentry{$\mathrm{S2S}$ with $S=2$}

    \addplot[
        color=\Sthree,
        mark options={solid},
        mark=*,
        very thick
    ]
    coordinates {
        (4,78.49139702)
        (10, 78.17710156)
        (20, 78.16761953)
        (50,78.19854844)
    };
    \addlegendentry{$\mathrm{S2S}$ with $S=4$}
    
    \end{axis}
  \end{tikzpicture}
\end{minipage}

\caption{Analysis of MISELBO estimation using the \modelname{S2S} estimator. (a) With the number of mixtures \( A \) set to a constant value, we incrementally increase the number of components \( S \), observing the impact on estimation accuracy. (b) Conversely, we maintain a constant number of components \( S \) while progressively increasing the number of mixtures \( A \) to assess the benefits of additional mixtures on the estimation process.}
\label{fig:increasing_S_for_s2s}
\end{figure}
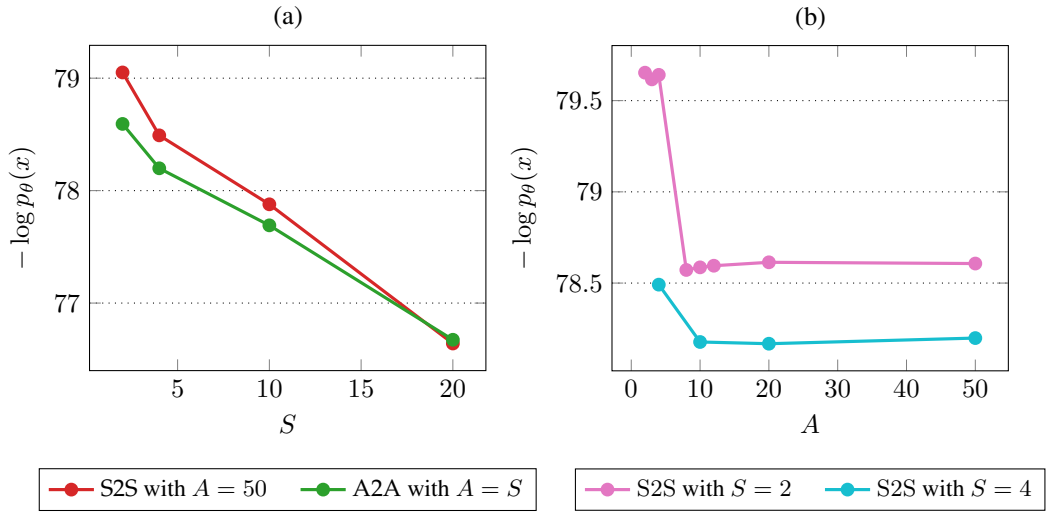

\subsection{Phylogenetics Experiment}

\subsubsection{Details of Variational Bayesian Phylogenetic Inference using mixtures and black-box variational inference}
\label{app:vbpi}
The task in Bayesian phylogenetic inference, is to approximate the posterior distribution over branch lengths, $\mathcal{B}$, and tree topologies, $\tau$, given the observed sequence data (typically DNA data), $x$. The phylogenetic posterior is thus defined as
\begin{equation}
    p(\tau, \mathcal{B}|x) = \frac{p(x| \tau, \mathcal{B})p(\mathcal{B}|\tau)p(\tau)}{p(x)},
\end{equation}
where the marginal likelihood, $p(x)$, is intractable.

Based on algorithms of \cite{zhang2018variational, zhang2019variational, zhang2020improved, zhang2022variational, zhang2023learnable, zhou2023phylogfn} we utilized the S2A and S2S to speed up the improvements made with mixtures in \cite{kviman2023improved}. The improvements follow the same structure as the original paper. In summary, the Subsplit Bayesian Networks (SBNs; \citet{zhang2018variational}) are utilized to learn tree topologies in a Bayesian phylogenetic context. An SBN employs a lookup table containing probabilities of subsplits (partial tree structures), referred to as a Conditional Probability Table (CPT). This table is learned through BBVI, and once the CPT is established, the SBN provides a tractable probability distribution over tree topologies, enabling sampling from this distribution. The parameters of SBNs are learned in the VBPI (Variational Bayesian Phylogenetic Inference) framework by maximizing MISELBO using VIMCO for variance reduction of the gradient. 

VIMCO (Variational Inference for Monte Carlo Objectives), the VBPI-Mixtures algorithm is a novel development in Bayesian phylogenetics. It demonstrates that mixtures of SBNs can approximate distributions unattainable by a single SBN, providing more accurate models of complex phylogenetic datasets. The VIMCO estimator, specifically derived for mixtures, enhances this approach. This estimator enables the VBPI-Mixtures algorithm to jointly explore the tree-topology space more effectively, leading to state-of-the-art results on various real phylogenetics datasets. Thus, mixtures of SBNs, coupled with the VIMCO estimator, significantly improve the accuracy of approximations of the tree-topology posterior in Bayesian phylogenetic inference. In this article, we took these improvements and combined them with the novel improvements of S2A and S2S to significantly speed up the training process with minimal to no loss in performance, resulting in a scalable solution that can be used for more advance phylogenetic problems.

\subsection{Generated Images}
\label{app:generated_imgs}
To assess the generative capabilities of our models, we have included visualizations of images generated from the MNIST and FashionMNIST datasets in Fig.~\ref{fig:gen_imgs}.

\begin{figure}
    \centering
\includegraphics{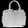 } \includegraphics{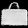 } \includegraphics{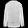 } \includegraphics{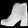 } \includegraphics{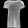 } \includegraphics{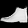 } \includegraphics{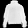 } \includegraphics{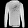 } \includegraphics{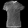 }
\\
\vspace{0.1cm}
\includegraphics{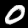}\includegraphics{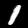} \includegraphics{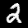} \includegraphics{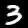} \includegraphics{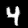} \includegraphics{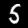} \includegraphics{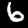} \includegraphics{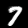} \includegraphics{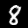} \includegraphics{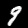}    \caption{\textbf{Top:} Images generated with a variational posterior with 400 mixture components (A=400), trained with the S2A estimator and S=1 on FashionMNIST. \textbf{Bottom:} Images generated with a variational posterior with 600 mixture components (A=600), trained with the S2A estimator and S=1 on MNIST}
    \label{fig:gen_imgs}
\end{figure}

\subsection{Comparable performance on CIFAR-10}
\begin{table*}[!htbp]
  \caption{NLL statistics for various SOTA VAE architectures on the \textbf{CIFAR}-10 dataset. The Composite model is a SEMVAE model which incorporates hierarchical models, NFs and the VampPrior. For IWAE $L$ is the number of importance samples used during training. }
  %\label{tab:NLL SOTA}
  \centering
  \begin{tabular}{lcc}
    \toprule

    Model     & NLL  \\
    \midrule
    
    NVAE \citep{vahdat2020nvae} & $2.93$\\
    PixelVAE++ \citep{sadeghi2019pixelvae++} & $2.90$\\
    MAE \citep{ma2019mae} & $2.95$& \\
    
    Vanilla SEMVAE ($S=12$; \citet{kviman2023cooperation}) & $4.83$\\
    
    CR-NVAE \citep{NEURIPS2021_6c19e0a6}  & $2.51$\\
    \midrule
    \modelname{MISVAE} $\mathrm{S2S}$ ($A=4, S=2$; \textbf{our})     & $3.23$\\
    \modelname{MISVAE} $\mathrm{S2A}$ ($A=4, S=2$; \textbf{our})     & $3.19$  \\
    \bottomrule
  \end{tabular}
\end{table*}

\newpage

\end{document}